\documentclass[%
  onecolumn
]{mpi2015-cscpreprint}

\usepackage[american]{babel}

\usepackage{amsmath,amssymb}
\usepackage{graphics,graphicx,epsfig,algorithm,algpseudocode}
\usepackage{graphicx}
\usepackage{acronym}
\usepackage{bookmark}
\usepackage{subfigure}
\usepackage{adjustbox}
\usepackage{csquotes}
\usepackage{dsfont}
\usepackage[font=small,labelfont=bf]{caption}
\usepackage{todonotes}
\usepackage{mathtools}

\theoremstyle{definition}
\newtheorem{definition}{Definition}[section]
\newtheorem{theorem}{Theorem}[section]

\newtheorem{lemma}[theorem]{Lemma}
\newtheorem{remark}[theorem]{Remark}

\def\bvcnn{\ensuremath{\mathbf{v}_{\mathsf{CNN}}}}

\def\bu{\ensuremath{\mathbf{u}}}
\def\bx{\ensuremath{\mathbf{x}}}
\def\tbv{\ensuremath{\tilde {\mathbf{v}}}}
\def\dtbv{\ensuremath{\dot{\tilde {\mathbf{v}}}}}
\def\bf{\ensuremath{\mathbf{f}}}

\def\bb{\ensuremath{\mathbf{b}}}
\def\bv{\ensuremath{\mathbf{v}}}
\newcommand\bvk[1]{\ensuremath{\mathbf{v}^{(#1)}}}
\newcommand\tk[1]{\ensuremath{t^{(#1)}}}
\newcommand\brk[1]{\ensuremath{\boldsymbol{\rho}^{(#1)}}}
\newcommand\plk[1]{\ensuremath{\mathbf l^{(#1)}}}
\def\bc{\ensuremath{\mathbf{c}}}
\def\bw{\ensuremath{\mathbf{w}}}

\def\bp{\ensuremath{\mathbf{p}}}

\def\bz{\ensuremath{\mathbf{z}}}
\def\by{\ensuremath{\mathbf{y}}}
\def\bl{\ensuremath{\mathbf{l}}}
\def\bPi{\ensuremath{\boldsymbol{\Pi}}}
\def\balpha{\ensuremath{\boldsymbol{\alpha}}}

\def\bbeta{\ensuremath{\boldsymbol{\beta}}}
\def\brho{\ensuremath{\boldsymbol{\rho}}}

\def\bU{\ensuremath{\mathbf{U}}}

\def\bV{\ensuremath{\mathbf{V}}}

\def\bA{\ensuremath{\mathbf{A}}}

\def\bJ{\ensuremath{\mathbf{J}}}
\def\bI{\ensuremath{\mathbf{I}}}
\def\bM{\ensuremath{\mathbf{M}}}
\def\bN{\ensuremath{\mathbf{N}}}
\def\bS{\ensuremath{\mathbf{S}}}
\DeclareMathOperator\sftmx{softmax}
\DeclareMathOperator\dist{dist}
\DeclareMathOperator\CO{Co}
\def\RE{\mathsf{Re}}
\def\Lrec{\mathcal L_{\mathsf{rec}}}
\def\Lclt{\mathcal L_{\mathsf{clt}}}
\def\Xcl{\mathcal X_{\mathsf{cl}}}

\newacro{POD}[POD]{\emph{Proper Orthogonal Decomposition}}
\newacro{AE}[AE]{\emph{autoencoder}}
\newacro{cPOD}[cPOD]{\emph{clustered POD}}

\begin{document}
  

\title{Polytopic Autoencoders with Smooth Clustering for Reduced-order Modelling
of Flows} 
  
\author[$\ast$,$\dagger$]{Jan Heiland}
\affil[$\ast$]{Max Planck Institute for Dynamics of Complex Technical Systems, Magdeburg, Germany}

\author[$\ast$,$\dagger$]{Yongho Kim}
\affil[$\dagger$]{Department of Mathematics, Otto von Guericke University Magdeburg,
    Magdeburg, Germany\authorcr
     \email{heiland@mpi-magdeburg.mpg.de}, \orcid{0000-0003-0228-8522}
    \email{ykim@mpi-magdeburg.mpg.de}, \orcid{0000-0003-4181-7968}
 }
  
\shorttitle{PAEs with Smooth Clustering for ROM of Flows}
\shortauthor{J. Heiland \& Y. Kim}
\shortdate{}
  

  
\abstract{
With the advancement of neural networks, there has been a notable increase, 
both in terms of quantity and variety, in research publications concerning the application of autoencoders to reduced-order models. 
We propose a polytopic autoencoder architecture 
that includes a lightweight nonlinear encoder, 
a convex combination decoder, and a smooth clustering network. 
Supported by several proofs, the model architecture ensures that all reconstructed states lie within a polytope, 
accompanied by a metric indicating the quality of the constructed polytopes, 
referred to as polytope error. 
Additionally, it offers a minimal number of convex coordinates for 
polytopic linear-parameter varying systems 
while achieving acceptable reconstruction errors compared to proper orthogonal decomposition (POD). 
To validate our proposed model, 
we conduct simulations involving two flow scenarios 
with the incompressible Navier-Stokes equation. 
Numerical results demonstrate the guaranteed properties of the model, 
low reconstruction errors compared to POD, 
and the improvement in error using a clustering network.
}

 \novelty{
 \begin{itemize}
   \item[] An autoencoder for PDE data with a linear decoder architecture based
  on the latent variable information and a smooth cluster selection algorithm.
\item[] Guarantee by design, that the reconstructed states lie within a polytope in the
  high-dimensional state space which is beneficial for polytopic \emph{quasi
  LPV} approximations of nonlinear PDEs.
\item[] Proof of concept and comparison to standard approaches for two numerical
  flow simulations.
\end{itemize}
 }

\maketitle

\section{Introduction}\label{sec:intro}
The solution of high-dimensional dynamical systems of the form
\begin{equation} \label{eq:gen_dynsys}
  \dot \bv(t) = f(\bv(t)), \quad\text{with }\bv(t) \in \mathbb R^{n}, \quad\text{for
  time }t>0,
\end{equation}
in simulations often demands substantial computational resources and even
become infeasible due to hardware constraints.
In response to these challenges, researchers have used model order reduction methods in diverse fields such as engineering, medicine, and chemistry (see e.g., \cite{SC22Tire,NL19MD,MK21Bio,VB20POD,MF23Brake}). 
The promise and procedure of model order reduction is to design a model or
reduced dimension to enhance computational efficiency while maintaining a
desired level of accuracy. 

\emph{Proper Orthogonal Decomposition (POD)} \cite{BG93POD}, a classical model
order reduction method, has gained wide acceptance due to its 
linearity, its optimality (as a
linear projection for given measurements of the state), and the advantages of
the POD modes as an orthogonal basis derived from a \emph{truncated singular value decomposition} (see e.g., \cite{NH11SVD}) of given data. 
POD provides the reduced coordinates in a low-dimensional space through a linear projection
\begin{equation*}
\bv_r(t)=\bV\bv(t)
\end{equation*}
and reconstructs the states using a linear lifting
\begin{equation*}
\tilde\bv(t)=\bV^\top\bv_r(t),
\end{equation*}
 where $\bV\in\mathbb{R}^{r\times n}$ is the matrix including the so-called
 \emph{leading} $r$ POD modes, where $\bv(t)$ is a $n$ dimensional state, where
 $\bv_r(t)\in\mathbb{R}^r$ are the reduced coordinates (by virtue of $r\ll n$), and
 where $\tilde\bv(t)$ is the reconstructed state that approximates $\bv(t)$.

The linearity of the POD comes with many algorithmic advantages but also means
a natural limitation in terms of accuracy versus reduction potential which is commonly known as the \emph{Kolmogorov n-width} \cite{OhlR16}.
For our intended application of very low-dimensional approximations, one,
therefore, has to resort to nonlinear model reduction schemes.

In this realm, general \emph{autoencoders} (see e.g., \cite{Dor20AE}) have been widely used either in conjunction with POD or as a substitute for POD, driven by two main reasons. 
Firstly, the universality: autoencoders can be constructed in various
architectures and contexts such as convolutional autoencoders, denoising autoencoders, and
variational autoencoders; see, e.g., \cite[Ch. 14]{Ian16DL} or \cite{DiMa14}.
Secondly, the efficiency: the training of autoencoders is a classical machine
learning tasks and, thus, comes with state-of-the-art implementations in all
machine learning toolboxes.
Accordingly, the research reports on autoencoders for reduced-order dynamical
systems have reached a significant
amount both in numbers and diversity in recent years; see e.g., \cite{PaRo20,EiHa20,Ph21,MaLuBa21,FreM22,BuGl23,XuLi24,DuHe24}.

Particularly, convolutional autoencoders for model order reduction have been developed 
due to the translation invariance and sparse connectivity of convolutions;
see e.g., \cite{MaLuBa21,FreM22,BuGl23,DuHe24}.
For these reasons and what has been reported so far, we choose to utilize convolutional autoencoders in this paper.

As another attempt to overcome the limits of linear model order reduction
schemes, one approach is to use local POD bases 
for state reconstruction on subregions of the given data set; cp. e.g.,
\cite{AmsH16,AmZa12,DuJo18,KaSh22}.
To define the subsets, often referred to as \emph{clusters}, one commonly refers
to  clustering methods such as \emph{$k$-means clustering} \cite{ArVa07} and fuzzy
$c$-means clustering \cite{Du73} and problem-specific measures.

In a previous work \cite{HeiKim22}, we have considered the direct combination of
the two approaches -- clustering and separate, i.e., local
autoencoders for each cluster.
The result was a highly-performant autoencoder for very low-dimensional parametrizations of incompressible flows.
However, the involved identification of the local bases is a highly nonlinear
and even noncontinuous process and, thus, not suited for use in dynamical
systems.

In the present work, following reported efforts on developing differential
clustering algorithms (see e.g., \cite{MiKe22})
we propose a fully differentiable clustering network suitable for large-scale dimensional systems.
In addition, since the model including the clustering network is fully
differentiable, 
the clustering model parameters and other model parameters can be trained simultaneously
through a joint loss function (see e.g., \cite{AuGa19,FaTh20,MiKe22}).

Another target of the presented research is motivated from the fact that, typically, the states of dynamical systems are confined to
a bounded subset of the state space which is not contradicting but in a sense
dismissing the premise of linear model order reduction that the states reside in
a linear space.
This additional feature is like naturally ensured in our proposed
architecture by defining the reconstruction as a convex combination of
supporting vectors associated with the smoothly selected clusters.
Furthermore, with standard tools of neural networks and with identifying
\emph{pseudo-labels} for the training, we can control the number
of involved supporting vectors which translates into very low-dimensional local
bases.

As an intended side effect, 
we reason that
the convex combination-based decoding readily defines an affine parametrization
within a polytope
as it can be exploited for efficient
nonlinear controller design in the context of 
\emph{linear parameter-varying} (LPV) systems; see
\cite{PiPa95} for robust controller design for polytopic LPV systems, 
and for applications, refer to \cite{JoPa06,SeHe11,TrRi16,SyFaJa18}.

The paper is structured as follows:
In \Cref{sec:moti}, we introduce the motivation and basic ideas for the application of autoencoders, convex polytopes, and clustering.
In \Cref{sec:notation}, we define notations and state basic properties convex polytopes.
In \Cref{sec:pae}, we introduce our proposed model, \emph{Polytopic Autoencoders (PAEs)} and define polytope error.
In \Cref{sec:sim}, we assess the reconstruction performance of PAEs in comparison to POD and CAEs. Additionally, we examine the outcomes achieved with pretrained PAEs.
\Cref{sec:concl} serves as the conclusion of our study, where we summarize our findings and provide insights into potential future research directions.

\section{Motivation and Basic Ideas}\label{sec:moti}

 In view of making the connections to \emph{autoencoders} and to the intended
 applications in \emph{linear-parameter varying} (LPV) approximations to
 dynamical systems as in \eqref{eq:gen_dynsys} later, we note the different nomenclature: In an
 autoencoder context, the reduced-order coordinates $\brho$ are referred to as
 \emph{latent variable}, whereas in an LPV context, we will consider the $\brho$ a
 parametrization of the states.

General nonlinear autoencoders of type
\begin{equation}
  \bv \to \mu(\bv) = \brho \to \varphi(\brho)= \tbv \approx \bv
\end{equation}
have been successfully employed for parametrizing dynamical systems like \eqref{eq:gen_dynsys} on a very low-dimensional (if compared to, e.g., POD) manifold via 
\begin{equation}
  \dtbv(t) = \frac{d\varphi}{d\brho}\dot\brho (t) = f(\varphi (\brho(t)));
\end{equation}
see e.g., \cite{LeeC20,Ph21,BuGl23}.
The nonlinear reconstruction $\tbv(t) = \varphi(\brho(t))$ together with the evaluation of the Jacobian $\frac{d\varphi}{d\brho}$ is computationally demanding let alone the ill-posedness of inferring an unstructured nonlinear map from the low-dimensional range of $\brho$ to the high-dimensional state space.
Because of these shortcomings, a decisive performance advantage of nonlinear model order reduction over linear projections is yet to be established. 

Using local linear bases to express $\tbv$, one enforces a certain structure that pays off both in less computational effort for the reconstruction and less unknown variables in the design of the decoder $\varphi$.
In our previous work \cite{HeiKim22}, we have confirmed that with individual affine linear
decoders for a-priori identified clusters in the latent space, superior
reconstruction can be achieved with the comparatively cheap operation of
locating the current value of $\brho$ in the correct cluster.

The use of clustering as well as any other selection algorithm for local bases,
however, comes at the cost of a 
noncontinuous decoding map $\varphi$. 
Therefore, the presented work aims at providing the reconstruction performance of
local bases but with a smooth selection algorithm so that
$\frac{d\varphi}{d\brho}$ and, thus, $\tbv$ is differentiable. 
For this, we base the decoding on the Kronecker product $\balpha \otimes \brho$ of the latent variable $\brho$ and a
smooth clustering variable $\balpha=c(\brho)$ and obtain the reconstruction
as a linear combination with $\balpha \otimes \brho$ as coefficients; see \Cref{fig:pae} for an
illustration.

\begin{figure}[t]
\centering
\includegraphics[width=0.8\columnwidth]{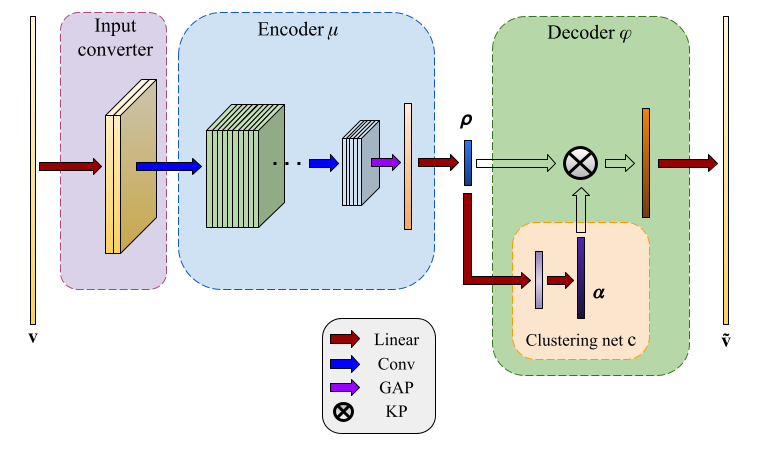}
\caption{Polytopic Autoencoder (PAE): \enquote{Linear} corresponds to a linear layer, \enquote{Conv} refers to a convolutional layer, \enquote{GAP} stands for global average pooling, and \enquote{KP} is the Kronecker product.}%
\label{fig:pae}
\end{figure}

By standard techniques from the training of neural networks, we will control the
number of nonzeros in $\balpha$ so that the reconstruction will happen with a
limited number of local basis vectors.
Even more, we can guarantee that the entries of $\balpha \otimes \brho$ are
nonnegative and sum up to one so that the reconstruction has the interpretation of
a convex combination within a polytope.

The intended manifold advantages are as follows:
\begin{itemize}
  \item The nonlinearity of the decoding is reduced to the map $\brho\to \balpha$
    between small dimensional sets.
  \item The decoding happens in a polytope defined by a small number of
    vertices.
  \item With $\balpha \otimes \brho$ being the coefficients of a convex
    combination with a small number of nonzeros, the reconstruction happens within
    a bounded set with improved stability due to the reduction of summations.
  \item The map $\balpha \otimes \brho \to \tbv$ is linear and provides an
    approximative polytopic expansion of the state $\bv$, which is useful in the
    numerical treatment of LPV systems.
\end{itemize}

Concretely, the proposed autoencoder comprises four key components:
 \begin{enumerate}
 \item A linear input converter that interpolates the spatially distributed data to a
   rectangular grid as it is needed for conventional convolutional layers in a
   neural network.
\item A nonlinear encoder which maps the high dimensional input states onto low
  dimensional latent variables with the final layer designed such that the
  latent variables are all positive and sum up to one (to later serve as
  coefficients for a convex combination (cp. \Cref{def:cc})) of supporting
  vectors.
\item A nonlinear smooth clustering network, responsible for locating the latent
  variables within one or more of $k$ clusters.
\item A linear decoder, which turns the latent variable and the clustering
  result into a reconstruction of a high-dimensional state within a polytope.
\end{enumerate}

In terms of equations, the map of an input $\bv$ to the reconstruction $\tilde \bv$
will read
\begin{subequations}
\begin{align*}
\bvcnn&=\bI_C\bv \\
\brho&=\mu(\bv) \\
\balpha&=c(\brho)\\
\tilde \bv &= \varphi(\brho)=\bU(\balpha\otimes \brho)
\end{align*}
\end{subequations}
where $\bv$ is the data, $\bvcnn$ is the input for a CNN,
$\bI_C$ is an interpolation matrix, $\bU$ is the matrix of all
supporting vectors, $\brho$ is the latent state, and $\tilde\bv$ is the
reconstruction designed to well approximate $\bv$. 
In view of applications to dynamical equations, we note that all involved
mappings are differentiable so that a differentiable (in time $t$) input results
in a differentiable output by virtue of
\begin{equation*}
\frac{d \tilde\bv}{dt} = \frac{d \varphi}{d \brho}\frac{d \mu}{d \bv}\frac{d \bv}{dt}
= \bJ_D\bJ_E\frac{d \bv}{dt}
\end{equation*}
where $\bJ_D$ and $\bJ_E$ are the Jacobian matrices of the decoder and encoder respectively.

\section{Preliminaries and Notation}\label{sec:notation}

We introduce basic notations and definitions concerning convex combinations and
polytopes and lay out details of the various components of the proposed
architecture.

\begin{definition}
  \label{def:cc}
Let $\mathcal{U}:=\{\bu_1,\bu_2,\cdots , \bu_n\}\subset \mathcal V$ be a finite
subset of a real vector space $\mathcal V$.
  \begin{enumerate}
    \item  Let $\lambda_1, \lambda_2, \cdots, \lambda_n$ be nonnegative real values
      satisfying $\sum_{i=1}^n\lambda_i=1$. Then $\sum_{i=1}^n \lambda_i \bu_i$ is
      said to be a \emph{convex combination} of the vectors $\bu_1,\bu_2,\cdots
      , \bu_n$ and
    \item the \emph{convex hull} of $\mathcal{U}$ is defined and denoted as
\begin{equation*}
\CO(\mathcal{U})=\{ \bz\in \mathcal V \,| \, \bz \text{ is a convex combination of the vectors of
} \mathcal U \} 
.
\end{equation*}
\item Moreover, if for any subset $\bar{\mathcal U}\subsetneq \mathcal U$ it
  holds that $\CO(\bar{\mathcal U}) \subsetneq \CO(\mathcal U)$, then we call
  $\CO(\mathcal U)$ a \emph{convex polytope}. 
\end{enumerate}
\end{definition}

\begin{remark}
  By construction, the convex hull $\CO (\mathcal U)$ is convex. This means that
  for any $z_1$, $z_2 \in \CO (\mathcal U)$ and $\lambda\in[0,1]$, we have $z=\lambda z_1 + (1-\lambda)z_2 \in \CO (\mathcal U)$.
  In what follows, we generally assume that $\CO(\mathcal U)$ is a convex
  polytope which basically means that the vectors in $\mathcal U$ are linearly
  independent.
\end{remark}

\begin{lemma}\label{lem:convc}
Let $\brho=[\rho_1,\rho_2, \cdots, \rho_r]^\top$ and
$\balpha=[\alpha_1,\alpha_2, \cdots, \alpha_k]^\top$ with $\rho_i\geq0$,
$i\in\{1,2,\cdots , r\}$, $\sum_{i=1}^r\rho_i=1$, $\alpha_j\geq0$,
$j\in\{1,2,\cdots , k\}$ and $\sum_{j=1}^k\alpha_j=1$. Then the entries
$\alpha_{ij}=\alpha_i\rho_j$ of $\balpha\otimes \brho \in \mathbb R^{kr}$ are
positive and sum up to one.
\end{lemma}
\begin{proof}
To begin, we note that by $\rho_i\geq 0$, $\alpha_j\geq 0$, we have that
$\rho_i\alpha_j$ is nonnegative for all $i$ and $j$. 
Moreover, for the sum over the elements of $\balpha\otimes \brho$, we have that
\begin{equation*}
 \sum_{i=1}^k \sum_{j=1}^r\alpha_i\rho_j=\sum_{i=1}^k \alpha_i
 \sum_{j=1}^r\rho_j =(\alpha_1+\alpha_2+\cdots + \alpha_k)(\rho_1+\rho_2+\cdots + \rho_r)=1.
\end{equation*}
\end{proof}

Throughout the paper, we adopt the following notation. Since we consider data
from dynamical systems, the involved variables $\bv$, $\bvcnn$, $\brho$,
$\balpha$ can be seen as functions of time $t$. Where appropriate, we will drop
these time dependencies and write, e.g., $\brho = \mu(\bvcnn)$. 
For general data points we will write, e.g., $\bv(t)$. Further down the text,
where we consider data that was sampled at time instances $t^{(k)}$, we write, e.g.,
$\bvk:=\bv(t^{(k)})$. A subscript like in $\rho_i(t)$ will denote the $i$-th
component of a vector-valued quantity or an enumerated set of items (like
columns of a matrix). 

\section{State Reconstruction within a Polytope}\label{sec:pae}
In this section, we provide detailed information about \emph{Polytopic Autoencoders (PAEs)} as illustrated in \Cref{fig:pae}.
Specifically, we demonstrate the process of feeding data generated by the Finite Element Method (FEM) to Convolutional Neural Networks (CNNs). 
We detail the design of a lightweight convolutional encoder architecture using
depthwise and pointwise convolutions, which significantly reduces the number of
model parameters compared to general full linear layers and also to POD. 
We also explain the methods for clustering reduced states in a low-dimensional space, 
reconstructing states within a polytope, 
training PAEs, 
and evaluating approximation errors in the polytopes that are defined by the PAEs.

\subsection{Input converter}
We consider data that comes from FEM simulations on possibly nonuniform meshes
which are not readily suited as inputs for convolutional neural networks. 
Rather than resorting to particular techniques like graph-convolutional neural
networks in this context (see e.g., \cite{PicMH23}),
we interpolate the data on a tensorized grid (cp. \cite{HeiBB22}) which is
efficiently realized through a very sparse (only 0.03\% nonzero entries)
interpolation matrix $\bI_C$.
As the data can be multivariate, we consider the interpolation result
 $\bvcnn(t)=\bI_C\bv(t)$ of a data point $\bv(t)$ as a three-dimensional tensor in $\mathbb R^{C\times
H \times W}$), where  $C$ stands for the number of input channels (i.e., the
dimension of the data) and $H$ and $W$ are the number of data points (i.e., the
number of pixels) in the two spatial dimensions.

\subsection{Encoder}\label{sec:encoder}
As the encoder 
\begin{equation*}
\mu\colon \mathbb{R}^{C\times H \times W}\rightarrow \mathbb{R}^{r}\colon
\bvcnn(t)\to \brho(t),
\end{equation*}
we set up a convolutional neural network with the final layer employing a
\emph{softmax} function that ensures for the output vector
$\brho(t)=\mu(\bvcnn(t))$ that all components are positive and sum up to one, i.e.,
\begin{equation}\label{eq:rho_cnvx_coeffs}
\rho_i(t)\geq 0, \quad \text{for }i=1,\dotsc,r \quad \text{and }\sum_{i=1}^{r}\rho_i(t)=1,
\end{equation}
and with, importantly, the reduced dimension $r$ being significantly smaller
than $n$. 

The \emph{softmax function} 
\begin{equation*}
  \bx \to \begin{bmatrix}
\frac{e^{x_i}}{\sum_{j=1}^{r} e^{x_j}}
\end{bmatrix}_{i=1,\dotsc,r} 
\end{equation*}
that is commonly used in machine learning to produce outputs in the form of a probability distribution 
appeared unsuited for the intended selection of vertices of a polytope later.
In fact, the standard \emph{softmax} never attains zero exactly and requires
rather large (in magnitude) input values for outputs close to zero.
That is why, we utilized a modified \emph{softmax} function
\begin{equation}\label{eq:msoftmax}
\sftmx(\bx)_i=\frac{x_i\cdot \tanh(10x_i)}{\sum_{j=1}^{r} x_j\cdot
\tanh(10x_j)}.
\end{equation}

As the signs of $x$ and $\tanh(x)$ always coincide, the (differentiable) function $x\to
x\cdot \tanh(10x)$ produces nonnegative values so
that the $\sftmx$ as defined in \eqref{eq:msoftmax} ensures the desired
property \eqref{eq:rho_cnvx_coeffs} for its output to potentially serve as
coefficients of a convex combination.

\begin{figure}[t]
\centering
\includegraphics[width=0.6\columnwidth]{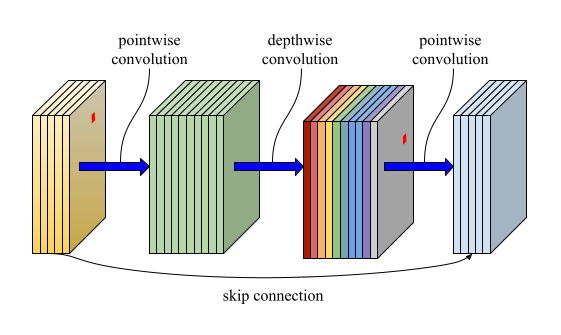}
\caption{Inverted residual block: an efficient approach for designing deep convolutional layers with fewer parameters compared to standard convolutions. (\Cref{sec:encoder})}%
\label{fig:encoder}
\end{figure}
\begin{remark}
  In view of treating high-dimensional data, the number of parameters in the
  neural networks is a pressing issue in terms of training efficiency but also
  memory requirements and computational efforts for forward evaluations.
To utilize a larger receptive field in the encoder with fewer parameters, we construct deep neural networks using convolution blocks that include a depthwise convolution and two pointwise convolutions as illustrated \Cref{fig:encoder}. A standard convolution operation extracts feature maps using a $K\times K\times C_O \times C_I$ kernel where $K$ represents the kernel size, $C_O$ is the number of output channels, and $C_I$ is the number of input channels. 
To soften possibly large memory requirements, one possibly may decompose the standard convolution into a depthwise convolution and a pointwise convolution, creating what is called \emph{depthwise separable convolution} \cite{MarAnd18}. 
We explain in \nameref{sec:app-dsc} how depthwise separable convolution utilizes fewer model parameters than the standard convolution.
\end{remark}

\subsection{Differentiable clustering network}\label{sec:clstr-net}
Naturally most popular clustering methods like the the $k$-means clustering
(see e.g., \cite{Scu10Kmeans})
classify data using discontinuous functions such as $\min$ and
$\max$ functions which are non-differentiable.
To alleviate this problem while maintaining the advantages of clustering for
reconstruction,
we resort
to a differentiable clustering network, that is implemented on the
$r$-dimensional latent space and defined as
\begin{subequations}\label{eq:clstrng-network}
\begin{align*}
c\colon \mathbb{R}^{r}\rightarrow \mathbb{R}^{k}\colon \brho(t) \to \balpha(t),
\end{align*}
\end{subequations}
where $c$ is a multi-layer \emph{perceptron} with the modified softmax function
$\sftmx(\bx)$, cp. \eqref{eq:msoftmax}, in the last layer.

In view of reconstruction within a polytope of a limited number of vertices, we
intend to ensure that $\balpha$ only has a small number of nonzero values.
Basically, a single nonzero entry of $1$ will correspond to a cluster
selection whereas a few nonzero entries will describe a smooth transition
between the clusters.

Although for the modified softmax functions, in theory, the range covers the
closed interval $[0,
1]$ the values will not be exactly zero or one in practice (except from a few rare cases).
Thus, we enforce decisive selections of clusters by including the selection of
pseudo labels for training the clustering network $c$; see \Cref{sec:tr}.

\subsection{Decoder}\label{sec:decoder}
We recall the approach of \emph{individual convolutional autoencoders (iCAEs)}
\cite{HeiKim22} that, for a single nonlinear encoder
\begin{equation*}
\brho=\mu(\bvcnn),
\end{equation*}
bases the reconstruction on $k$ individual (affine) linear decoders
\begin{equation*}
\tilde \bv=\bU_{l}\brho+\bb_{l}, \, l=1,2,\cdots , k
\end{equation*}
on $k$ clusters. Here, $\bU_l\in\mathbb{R}^{n\times r}$ and
$\bb_l\in\mathbb{R}^{n}$ are a matrix of local basis vectors and a bias for the reconstruction of states in the $l$-th cluster. 

If one leaves aside the bias terms,
then the reconstructed states $\tilde \bv$ can be described as a discontinuous decoder 
\begin{equation*}
\tilde \bv=\sum_{i=1}^{k}\beta_i\bU_i\brho=\sum_{j=1}^{r}\sum_{i=1}^{k}\beta_i\rho_j\bu_{i;j}
\end{equation*}
where $\beta_i$ is $i$-th element of a vector
\begin{equation*}
\bbeta = \begin{cases*}
  1 & if $i = l$,\\
  0            & if $i\neq l$
\end{cases*}
\end{equation*}
generated by $k$-means clustering 
(i.e., all its elements are 0 except for a single element which has a value of 1), and $\bu_{i;j}$ is the $j$-th column vector of $\bU_i$. 
This decoder is to select an individual decoder depending on the cluster which
is a nonsmooth operation where the states leave one cluster for another. 

Here, we replace the selection vector $\bbeta$ by the output $\alpha$ of a
smooth clustering network $c$ that allows for several nonzero entries and, thus,
enables smooth transitions between clusters:
\begin{subequations}\label{eq:decoder}
\begin{align*}
\tilde \bv&=\sum_{j=1}^{r}\sum_{i=1}^{k}\alpha_i\rho_j\bu_{i;j} \\
&=\bU(c(\brho)\otimes \brho)
\end{align*}
\end{subequations}
where $\alpha_i$ is $i$-th element of the smooth clustering output
$\balpha=c(\brho)$,
where $\otimes$ is the Kronecker product, and where
\begin{equation*}
\bU=
\begin{bmatrix}
| & | &  & | & & | & & |\\
\bu_{1;1} & \bu_{1;2}& \cdots & \bu_{1;r} & \cdots & \bu_{k;1}& \cdots & \bu_{k;r}\\
| & | &  & | & & | & & |\\
\end{bmatrix}
\end{equation*}
is the matrix of all supporting vectors for the reconstruction.

\begin{remark}\label{rem:reconst_polytope}

Due to the modified softmax function in the last layer of $\mu$ and $c$, it
holds that $\brho(t) = \mu (\bvcnn(t))$ and $\balpha(t) = c(\brho(t))$ oth
satisfy property \eqref{eq:rho_cnvx_coeffs},
so that by Lemma \ref{lem:convc}, the decoder output $\varphi(\brho(t))$ is a convex combination of the column vectors of $\bU$.
Consequently, all reconstructed states are generated within a
polytope $\tilde{\mathcal{V}}$ defined by the vertices $\bu_{1;1}, \dotsc , \bu_{k;r}$.
\end{remark}

\subsection{Training Strategy for PAEs}\label{sec:tr}
\begin{figure}[t]
\centering
\includegraphics[width=0.5\columnwidth]{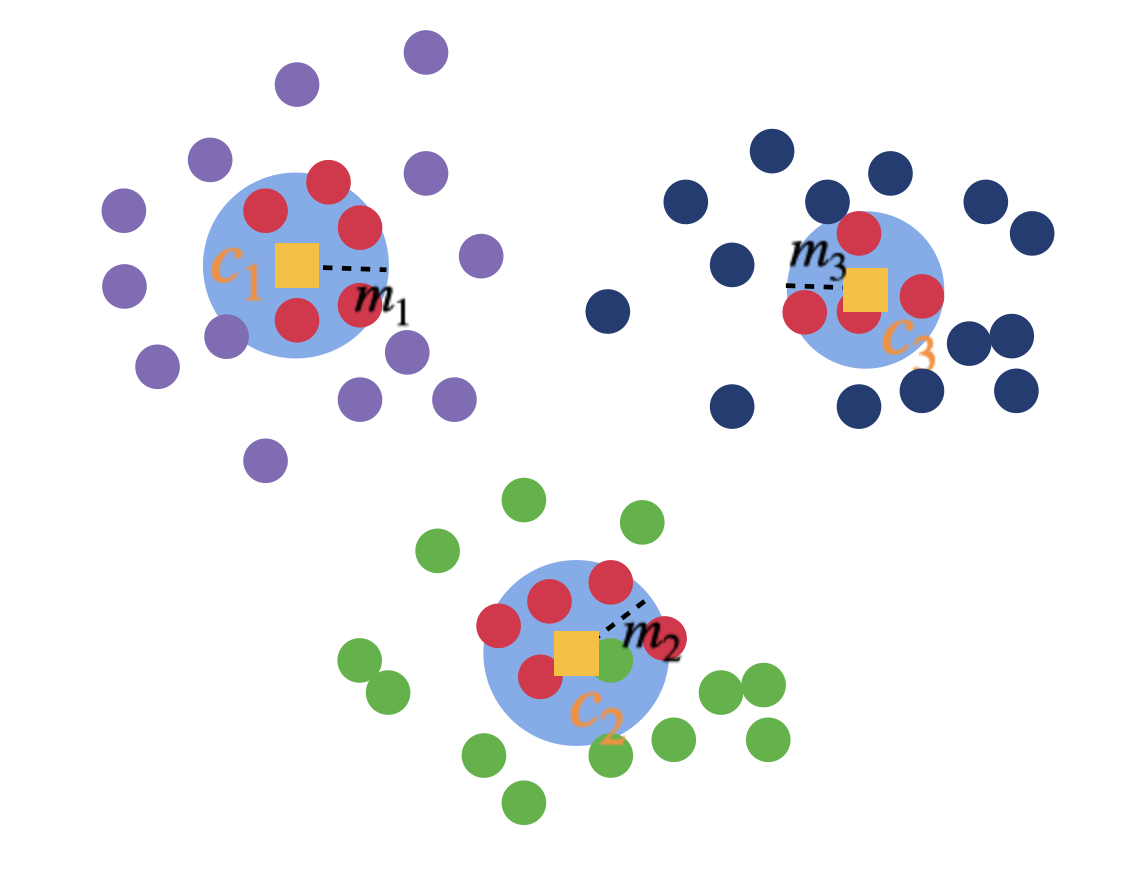}
\caption{When latent variables are divided into three clusters in a low-dimensional space, the clustering labels corresponding to the latent variables (red circles) within each circle with radius $m_i$, $i=1,2,3$, are chosen as target labels. Unselected labels are not used for training PAEs. (\Cref{sec:tr})}
\label{fig:tlabels}
\end{figure}

The proposed encoder, decoder, and clustering network are fully differentiable
with respect to input variables and all model nodes, enabling a joint
optimization of all model parameters using gradient-based techniques (e.g., Adam \cite{Kin15ADAM}). 
Nonetheless, in view of training efficiency and since the smooth clustering
network will be trained by means of pseudo-labels obtained from $k$-means
clustering, the overall training strategy includes two preparatory steps. 
In a final step, all components are then fine-tuned in a joint optimization.

We comment on these three training steps with details deferred to the
technical description in \nameref{sec:app-train}.

\textbf{Step 1} (initialization and identification of the latent space): We train a CAE consisting of a nonlinear encoder $\mu$ and a polytopic decoder $\bar\varphi$.

\textbf{Step 2} ($k$-means clustering for generating pseudo-labels and
initialization of individual decoders): From the pretrained encoder, we obtain
the latent variable coordinates for the training data which are then used for
$k$-means clustering in the latent space. The clustering results are then 
used both for generating target labels for a supervised training of the smooth
clustering network $c$ and for training individual decoders in the clusters.

In order to avoid overfitting of the $k$-means results and to shift the focus
away from the cluster centers to the region around them, we assign the pseudo
labels to represent areas around the centroids and train the clustering network
to best match the pseudo-labels in a distributional sense.

Concretely, for each cluster $i$, we select those $j(i)$ data points for which the
latent coordinates satisfy
\begin{equation}\label{eq:targetlabels}
  \| \brho^{(j(i))}-\bc_i \| <m_i, \, i=1,2,\cdots , k
\end{equation}
where $\bc_i$ is the centroid of the $i$-th cluster and where $m_i$ is the mean of the
distances between the latent variables and $\bc_i$ in this cluster.
To these selected datapoints we assign the $i$-th unit vector as the
pseudo-label.

With this procedure repeated for all clusters, we collect a set of data/labels
pairs
\begin{equation}\label{eq:xcl} 
\Xcl :=\{(\bv^{(1)}, \plk 1), (\bv^{(2)}, \plk 2),
\dotsc ,(\bv^{(N_l)}, \plk{N_l})\}
\end{equation}
where $N_l$ is the overall number of data
points selected by the criterion \eqref{eq:targetlabels}, where $\bvk{i}$ is the
velocity data so that $\brk{i}=\mu(\bvk{i})$, and where the labels $\plk{i}$
are unit vectors in $\mathbb R^{r}$ representing the corresponding cluster.

In a next preparatory step, the clustered and labelled data $\mathcal X_{\mathsf{cl}}$ is
used to train individual convex combination-based decoders $\varphi_{1}, \cdots,
\varphi_{k}$ on the $k$ clusters.

Then, the matrix  $\bU=[\theta_{\varphi_{1}}, \cdots, \theta_{\varphi_{k}}]$
that collects all supporting vectors of the individual decoders is used to
initialize the weights of a global, smooth and clustering-based decoder.

\textbf{Step 3} (training of the clustering net and fine-tuning of the PAE): we
train a PAE by fine-tuning $\mu$ and $\bU$ while simultaneously optimizing the
clustering network $c$.
For the model optimization, we define a reconstruction loss as
\begin{equation*}
\Lrec = \frac{1}{|B|}\sum_{i\in B}\parallel \tilde \bv^{(i)} - \bv^{(i)}\parallel_\bM
\end{equation*}
where $B$ is the
index set of a data batch drawn from the training data.
This loss function is the standard \emph{mean squared error} loss but in the
$\mathbf M$-norm that reflects the PDE setup.
Also, we define a clustering loss as the cross entropy loss 
\begin{equation*}
  \Lclt = -\frac{1}{|P|}\sum_{j\in P\subset B} \plk{j}\cdot \log(c(\brk{j})).
\end{equation*}
where $j$ comes from that subset $P\subset B$ that contains only those indices that
address data that is part of the clustered and labelled data set $\Xcl$ too; cp. \eqref{eq:xcl}.
The cross-entropy loss function describes the distance 
between two probability distributions 
and is commonly used as a loss function 
for training multi-class classification machine learning models. Note that both
the labels $\plk{j}$ (as unit vectors representing a sharp uni-modal
distribution) and the clustering output $\balpha^{(j)}=c(\brk{j})$ (by virtue of the $\sftmx$~in the final
layer; cp. \Cref{sec:clstr-net}) represent probability distributions.

Finally, we consider the joint loss function
\begin{equation*}
  \mathcal{L} =\Lrec +10^{-4}\Lclt,
\end{equation*}
where the weight $10^{-4}$ turned out to be a good compromise between accuracy
and overfitting.

\subsection{Polytope Error and Polytopic LPV Representation}\label{sec:polyerr}
\begin{figure}[t]
\centering
\includegraphics[width=0.35\columnwidth]{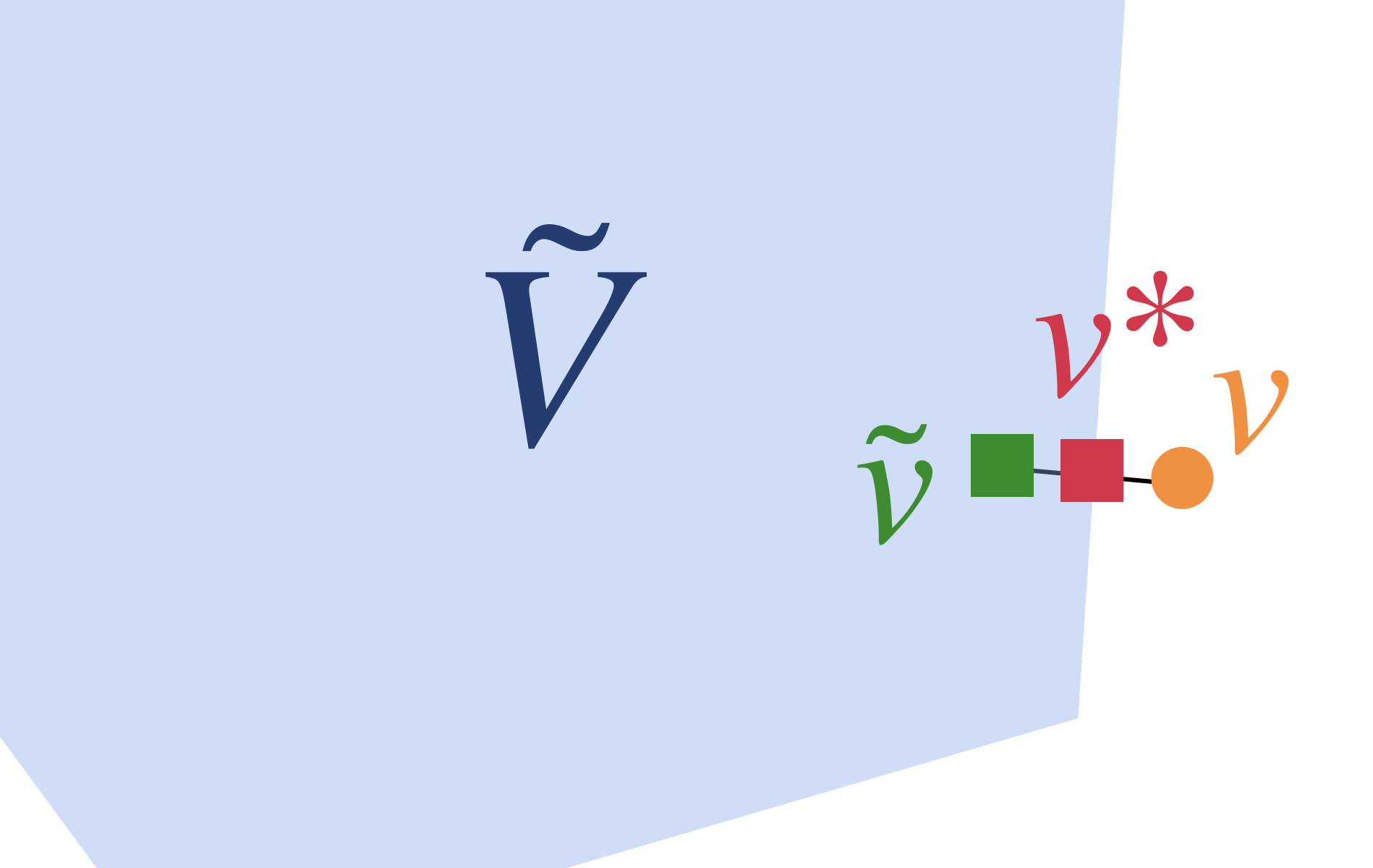}
\caption{Conceptual figure depicting the positions of a state $\bv$, it's
reconstruction $\tbv$ in a polytope, and it's best approximation $\bv^*$. (\Cref{sec:polyerr})}%
\label{fig:vstar}
\end{figure}

By design, the reconstruction $\tbv$ is generated inside a polytope, cp.
\Cref{rem:reconst_polytope}. We denote this polytope by $\mathcal V$.

\begin{definition}[Polytope error and best approximation]\label{def:pe}
  Let $\mathcal{V} \subset \mathbb R^{n}$ be a convex polytope. For a given data
  point $\bv\in \mathbb R^{n}$, let $\|\cdot \|_{\mathbf M}$ be the norm that is
  induced by the $\mathbf M$-weighted inner product and let
  \begin{equation*}
    \dist_{\mathbf{M}}(\bv, \mathcal V) := \min_{\bw\in \mathcal V}\|\bv-\bw\|_\mathbf{M}
  \end{equation*}
  be the \emph{polytope error} and let $ \bv^\ast\in{\mathcal{V}} $ be the
  \emph{best approximation} that realizes the minimum.
\end{definition}

We note that as shown in the appendix in \Cref{lem:dist}, the \emph{polytope
errors} and the \emph{best approximation} is well defined; see also \Cref{fig:vstar}
that illustrates the conceptual representation of the polytope error. 

In the numerical experiments, for the data at hand, we will evaluate the polytope error for the
polytope that is identified and used for the encoding and the reconstruction by
the PAE.
This provides best-case estimates for
\begin{enumerate}
  \item how well the identified polytope can represent the data and
  \item how close the reconstruction gets to this theoretical lower bound.
\end{enumerate}

The computation of the best approximation is a nontrivial task.
To compute the \emph{polytope error} for a given $\bv$, we solve the optimization problem
\begin{subequations}
\begin{align*}
  \min_{\rho \in \mathbb R^{r}} \|  \bU\brho-\bv\|_\bM^2\\
&\text{subject to}\,\, \phantom{\bI_r}\brho\geq 0,\\
&\quad\quad\quad\quad\,\,\,\mathds{1}_r\brho= 1,
\end{align*}
\end{subequations}
for the coordinates $\brho^\ast$ of the best approximation $v^\ast = \bU\brho$,
where $\bU$ is the matrix of the vertices (cp. \Cref{rem:reconst_polytope}),
and where $\mathds{1}_r=[1,1,\cdots
, 1]$ is the row vector of ones with $r$ entries. 

This convex quadratic programming problem involving linear constraints does not
have closed-form solution but numerical methods including active-set methods and interior-point methods can efficiently find solutions.
In the experiments, we use a quadratic programming solver 
provided by the Python library \texttt{cvxopt} which implements interior-point methods \cite{Ma11OPT}.

\subsection{Application in Polytopic LPV Approximations}\label{sec:lpvapprox}

We briefly comment on the intended application in approximating general
nonlinear functions $f\colon \mathbb R^{n}\to \mathbb R^{n}$ by so-called
linear-parameter varying approximations of preferably low parameter dimension. 
Such approximations are a promising ingredient for nonlinear controller design;
see \cite{DasH23} for the basic theory and proofs of concepts. 

An intermediate step is the representation of $f$ in state-dependent coefficient
form
\begin{equation*}
 f(\bv) = A(\bv)\,\bv
\end{equation*}
which always exists under mild regularity conditions and in the case that
$f(0)=0$. 
If then $A(\bv)\approx A(\tbv)$ is approximated by the autoencoded state
$\tbv = \mu(\brho(\bv))$, an LPV approximation with $\brho=\brho(\bv)$ as the
parameter is obtained by means of
\begin{equation*}
  f(\bv) \approx A(\tbv)\,\bv = A(\mu(\brho(\bv))\,\bv =: \tilde A(\brho) \, \bv.
\end{equation*}

As so-called affine linear polytopic LPV representations are of particular
use for controller design, we show how a polytopic reconstruction with $\balpha \otimes \brho$ 
transfers to a polytopic LPV approximation with now $\balpha \otimes \brho
$ as the parameters.

\begin{lemma}\label{lem:lpv}
  Let $\bA(\cdot)\colon \mathbb R^{n}\to \mathbb R^{n\times n}$ be a linear map
  and let $\tilde{\bv}(t)$ be a (convex) combination of $n$ vertices of a polytope.
  Then $\bA(\tilde{\bv}(t))$ is a (convex) combination of $n$ matrices
  representing a polytope in $\mathbb R^{n\times n}$.
\begin{proof}
Since $\tilde{\bv}(t)$ is a (convex) linear combination of $n$ vertices, $\tilde{\bv}(t)$ can be expressed as
\begin{equation*}
\tilde{\bv}(t)=\sum_{i=1}^{n}\zeta_i\bu_i
\end{equation*}
where $\zeta_i \in \mathbb R^{}$ ($\zeta_i \geq0$, $\sum_{i=1}^{n}\zeta_i=1$) and $\bu_i$ is $i$-th vertex of a polytope, $i\in\{1,2,\cdots , n\}$. Then
\begin{equation*}
 \bA(\tilde{\bv}(t))=\bA(\sum_{i=1}^{n}\zeta_i\bu_i)
\end{equation*}
Since $\bA$ is linear in it's argument, we have that
\begin{equation*}
 \bA(\tilde{\bv}(t))=\bA(\sum_{i=1}^{n}\zeta_i\bu_i)=\sum_{i=1}^{n}\zeta_i\bA_i
\end{equation*}
where $\bA_i:=\bA(\bu_i)$.
Therefore, $\bA(\tilde{\bv}(t))$ is is a (convex) linear combination of $\bA_1, \bA_2, \cdots , \bA_n$. 
\end{proof}
\end{lemma}

The alternative and standard way of locating an affine LPV representation of parameter dimension
$r$ in a polytope is to
determine the $r$-dimensional bounding box; see e.g., \cite{KwiW08}.
Here, the number of the vertices of the polytope is $2^r$, and thus, increases
exponentially with the reduced dimension $r$.
If one succeeds to identify a bounding polytope of less vortices, one is
confronted with computing the coordinates within this coordinates for which
there is no established method for dimensions beyond $r=3$; cp. the discussion
in \cite{DasH23}. 

In both respects, PAEs offer a promising alternative. Firstly,
the relevant parametrization $\balpha \otimes \brho$ already defines the needed
bounding polytope with a size of $r \cdot k$ which grows linearly in the number of clusters
$k$ and the dimension of the parametrization. Note that the difference to $2^r$
becomes advantageous for moderate and large $r$, meaning that, e.g., $r\cdot k < 2^r$ for $r\geq 5$ and $k\leq r$.
Secondly, the decoder readily provides the coordinates in the consider polytope.

\section{Simulation Results}\label{sec:sim}
As an example application, we consider flow simulations with the \emph{incompressible Navier-Stokes equations}
\begin{subequations}\label{eq:nse}
\begin{align}
\frac{\partial}{\partial t}\mathrm{v} + (\mathrm{v}\cdot\nabla)\mathrm{v}+
\frac{1}{\RE}\Delta \mathrm{v} - \nabla \mathrm{p} &= \mathrm{f} \\
\nabla \cdot \mathrm{v}&=\textbf{0},
\end{align}
\end{subequations}
where $\mathrm{v}$, $\mathrm{p}$, and $\mathrm{f}$ are the velocity, pressure,
and forcing term respectively and where $\RE$ is the Reynolds number.
After an FEM discretization of Equation (\ref{eq:nse}) a semi-discrete model is
obtained as 
\begin{subequations}\label{eq:nse-disc}
\begin{align}
\bM\dot\bv(t) + \bN(\bv(t))\bv(t)+ \bA\bv(t) - \bJ^\top\bp(t) - \bf(t) &= \textbf{0} \\
\bJ\bv(t)&=\textbf{0},
\end{align}
\end{subequations}
where $\bv(t)\in\mathbb{R}^{n}$ and $\bp(t)\in\mathbb{R}^{p}$ are the states of the velocity and the pressure at time $t$ respectively.
and where $\bM\in\mathbb{R}^{n\times n}$, $\bN(\cdot)\in\mathbb{R}^{n\times n}$,
$\bA\in\mathbb{R}^{n\times n}$, $\bJ\in\mathbb{R}^{p\times n}$, and
$\bf(t)\in\mathbb{R}^{n}$ are the mass, convection, diffusion, discrete
divergence matrices, and the forcing term at time $t$ respectively; see the
details for realizing the convection as a state-dependent coefficient in \cite{BehBH17}.

For the presentation that follows, we employ the so-called ODE formulation of
the DAE \eqref{eq:nse-disc} leaving aside
all technical challenges associated with handling the pressure $\bp(t)$ in
numerical schemes; cp. \cite{AltH15}.
The ODE formulation of (\ref{eq:nse-disc}) is derived under the reasonable
assumption that that $\bM$ and $\bJ\bM^{-1}\bJ^\top$ are invertible and the
observation that $\bJ\bv(t)=\textbf{0}$ implies
\begin{equation*}
\bJ\dot\bv(t)=\textbf{0}.
\end{equation*}
Then, by multiplying $\bM^{-1}$ on both sides of Equation (\ref{eq:nse-disc}), we obtain
\begin{equation*}
\dot\bv(t) + \bM^{-1}(\bN(\bv(t))\bv(t)+ \bA\bv(t) - \bJ^\top\bp(t) - \bf(t)) = \textbf{0}
\end{equation*}
Then, by multiplying $\bJ$ on both sides,
\begin{equation*}
\bJ\bM^{-1}(\bN(\bv(t))\bv(t)+ \bA\bv(t) - \bJ^\top\bp(t) - \bf(t)) = \textbf{0}\, (\because \bJ\dot\bv(t)=\textbf{0})
\end{equation*}
Finally, $\bp(t)$ can be described with respect to $\bv(t)$ and $\bf(t)$ as follows:
\begin{equation*}
\bp(t) =\bS^{-1}\bJ\bM^{-1}(\bN(\bv(t))\bv(t)+ \bA\bv(t) - \bf(t))
\end{equation*}
where $\bS=\bJ\bM^{-1}\bJ^\top$.
Thus, we can eliminate the pressure $\bp(t)$ from (\ref{eq:nse-disc}) and the equation can be presented as
\begin{equation}\label{eq:nse-ode}
\bM\dot\bv(t) = \bPi^\top(\bN(\bv(t))\bv(t)+ \bA\bv(t) - \bf(t))
\end{equation}
where $\bPi=\bM^{-1}\bJ^\top\bS^{-1}\bJ-\bI$. All data are generated by Equation (\ref{eq:nse-ode}).

\subsection{Data Acquisition and Performance Measures}\label{sec:measure}

For different setups we perform simulations of the FEM discretized incompressible flow
equations \eqref{eq:nse-disc} over time $t$ starting from the associated \emph{Stokes} steady state
and collect the data for training the autoencoders.
The data points are then given as the \emph{snapshots} of the (discrete)
velocity variable $\bvk{k}=\bv(\tk{k})$ at time instances $\tk{k}$.

Using this data, we compute POD approximations and optimize the neural networks
according to the following performance criteria.

In each reduced dimension of $r$, we evaluate the reconstruction performance of PAEs compared to other methods by measuring the averaged relative error
\begin{equation*}
\frac{1}{T}\sum_{i=1}^T\frac{\parallel
\tilde{\bv}^{(i)}-\bv^{(i)}\parallel_\bM}{\parallel \bv^{(i)}\parallel_\bM},
\end{equation*}
and the averaged relative polytope error $\varepsilon_{p}$ (cp. \Cref{def:pe})
\begin{equation*}
\frac{1}{T}\sum_{i=1}^T\frac{\parallel \bv^{(i)\ast}-\bv^{(i)}\parallel_\bM}{\parallel \bv^{(i)}\parallel_\bM}
\end{equation*}
where $T$ is the number of snapshots.

Additionally, we investigate the trajectories of reconstructed states and latent variables, as well as the polytopes used for the reconstruction.

In view of memory efficiency, we report the number of encoding parameters, decoding parameters, and vertices of polytopic LPV representations.
The number of encoding parameters for CAE and PAE is much less than that of POD.
This reduction in size is attributed to the immutable and sparse interpolation matrix $\bI_C$ and depthwise separable convolutions.
Regarding the number of decoding parameters, for POD and CAE, their decoders are linear, resulting in sizes of $nr$, denoting the number of elements for a $n \times r$  matrix. 
the decoding size of PAE is $nrk+m$ where $m$ is the number of parameters in the clustering net.

As an additional performance characteristics, we report the number $R$ of
vertices of a polytope that contains the reconstruction values as it would be
used for LPV approximations; cp. \Cref{sec:lpvapprox}. 
For the POD approximations, we consider the bounding box with $R=2^r$, where $r$
is the dimension of $\brho$ which is the standard approach in absence of, say,
an algorithm that would compute a polytopic expansion in a general polytope.

For the CAE or the PAE, however, this polytopic expansion is readily given in a
polytope of $R=r$ or $R=kr$ vertices, where $k$ is the number of clusters. Note that due
to summation condition $\sum_{i=1}^r \rho_i(t)=1$, the dimension of the
polytopes for CAE is even one less than what one can expect in the general case.

\subsection*{Data availability}
The source code of the implementations used to compute the presented results is available from 
  \href{https://doi.org/10.5281/zenodo.10491870}{\texttt{doi:10.5281/zenodo.10491870}}
under the Creative Commons Attribution 4.0 international license and is authored by Yongho Kim.

\subsection{Dataset: single cylinder}
Our data are generated in the time domain $[0,16]$. 
We use a Reynolds number of 40 for the single cylinder case.
Each snapshot vector $\bv(t)$ has 42,764 states (i.e., $n=42764$) associated with nodes of the FEM mesh in the spatial domain $(0,5)\times (0,1)\subset \mathbb R^{2}$. 

The dataset is divided into a training set containing 500 snapshots in the time interval $[0,10]$ and a test set including 300 snapshots in $[10,16]$.
When convolutional encoders are employed, the interpolation matrix
$\bI_c\in\mathbb{R}^{42764\times 5922}$ maps $\bv(t)\in\mathbb{R}^{42764}$ into
$\bvcnn(t)\in\mathbb{R}^{2 \times 63\times 47}$ consisting of the
$x$-directional velocity and the $y$-directional velocity values at each mesh
point on a rectangular grid of size $63\times 47$.

\subsection{PAEs: single cylinder}\label{sec:single-pae}
\begin{table}[t]
\begin{adjustbox}{width=200pt,center} 
\begin{tabular}{c|c|c|c|c|c|c}
\textbf{Model} & $r$ & $L_e$ & $P_e$ & $L_d$ & $P_d$ & $R$ \\
\hline
POD & 2 & 1& 85,528 &1& \textbf{85,528}  & 4 \\ 
CAE & 2 &15& \textbf{36,692} &1& \textbf{85,528}  & \textbf{2}\\ 
PAE(k=3) & 2 &15& \textbf{36,692}&3 & 256,584 & 6  \\
\hline
POD & 3 &1& 128,292&1 & \textbf{128,292}  & 8 \\ 
CAE & 3 &15& \textbf{36,725}&1 & \textbf{128,292}  & \textbf{3} \\ 
PAE(k=3) & 3 &15& \textbf{36,725}&3 & 384,876  & 9 \\
\hline
POD & 5 &1& 213,820&1 & \textbf{213,820}  & 32 \\ 
CAE & 5 &15& \textbf{36,791}&1 & \textbf{213,820}  & \textbf{5} \\ %
PAE(k=3) & 5 &15& \textbf{36,791}&3 & 641,460 & 15 \\
\hline
POD & 8 &1& 342,112 &1& \textbf{342,112}  & 256 \\ 
CAE & 8 &15& \textbf{36,890}&1 & \textbf{342,112} & \textbf{8} \\ %
PAE(k=3) & 8 &15& \textbf{36,890}& 3 & 1,026,336  & 24 \\
\end{tabular}
\end{adjustbox}
\caption{Model information: the number of encoding ($L_e$) and decoding ($L_d$) layers, the number of encoding ($P_e$) and decoding ($P_d$) parameters, 
and the number
$R$ which counts the vertices of a bounding box in $\mathbb R^r$ (for POD) or of
the polytopes used for the reconstruction for CAE and PAE for the single cylinder case (\Cref{sec:single-pae});
see how to calculate $P_e$ and $P_d$ in \Cref{sec:measure}.}\label{tab:single-params}
\end{table}

In this simulation, each PAE and CAE has a deep convolutional encoder consisting of 14 convolutional layers and a fully connected layer. 
We use the ELU activation function \cite{Djo16ELU} in the convolutional layers and the modified softmax function (\ref{eq:msoftmax}) in the last layer.
To reduce the number of nodes in the last layer, the global average pooling is used before performing the fully connected computation.
The decoder of POD is a linear combination of $r$ vectors and the decoder of CAE is a convex combination of $r$ vertices. The PAE decoder is partially linear as mentioned in \Cref{sec:decoder}.
In other words, CAE consists of an encoding part
\begin{subequations}
\begin{align*}
\bvcnn(t)=\bI_C\bv(t) \\
\brho(t)=\mu(\bvcnn(t)) 
\end{align*}
\end{subequations}
and a decoding part
\begin{subequations}
\begin{align*}
\tilde \bv(t) &= \varphi(\brho(t))
\end{align*}
\end{subequations}
without clustering. 
Consequently, CAE is regarded as PAE with 1 cluster (i.e., $k=1$).

Table \ref{tab:single-params} presents a comparison of the number of encoding and decoding parameters.
CAE and PAE maintain a relatively consistent number of encoding parameters regardless of reduced dimensions, in contrast with POD.
In practice, when $r=2,3,5,8$, the encoding size of CAE and PAE is only $42.6\%, 28.6\%, 17.2\%$, and $10.8\%$ of the encoding size of POD respectively in terms of the number of encoding parameters.
The decoding size of POD and CAE is decided by a $n \times r$ decoding matrix. 
In contrast, the decoding size of PAE is larger than them due to the Kronecker product of $\balpha$ and $\brho$.

For training CAEs and PAEs, the Adam optimizer is used with a learning rate $\eta$ of $10^{-4}$, a batch size of 64, and a clustering loss weight $\lambda$ of $10^{-4}$; see the details in \nameref{sec:app-train}.

\begin{figure}[t]
\centering
\includegraphics[width=0.8\columnwidth]{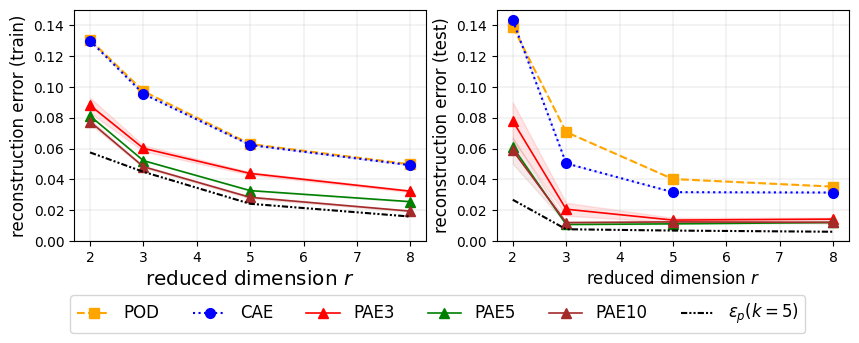}
\caption{Reconstruction error across the reduced dimension $r$ averaged for 5
runs for the single cylinder case (\Cref{sec:single-pae}). The shaded regions
mark the statistical uncertainty measured through several training runs and
appears to be insignificant and, thus, invisible in the plots for most methods.}%
\label{fig:rec-single}
\end{figure}
\subsection{Results: single cylinder}

In this section, we investigate how PAEs reconstruct periodic flows and handle their latent variables in very low-dimensional spaces. \Cref{fig:rec-single} shows the reconstruction errors of POD, CAE, and PAE against the reduced dimension $r$. These errors are calculated using the averaged errors obtained from 5 training trials, resulting in very small standard deviations. The CAEs achieve similar errors to POD for the training data in the time range $[0,10]$ and outperform POD in terms of the test reconstruction errors except for the error at $r=2$. 

It shows that the CAEs reconstruct periodic flows in $[10,16]$ better than POD.
Overall, the reconstruction performance of the CAEs is comparable to POD, although the CAEs utilize fewer model parameters for the reconstruction and determine a smaller $R$ for the design of polytopic LPV systems.
The PAE(k)s that cluster latent variables into $k$ clusters (e.g. PAE3, PAE5, PAE10) outperform the CAEs and POD. 
The reconstruction errors of the PAEs tend to be reduced as $k$ gets larger. 
However, there is no significant gap between the errors of PAE5 and PAE10.

As a result of the polytope errors with $k=5$, they are less than $2.7\%$ for the periodic flows in the testing range $[10,16]$. 
Specifically, these errors depending on the reduced dimensions $r=2,3,5,8$ achieve $2.7\%, 0.8\%, 0.7\%,$ and $0.6\%$ respectively.
It indicates that the polytopes defined by the PAE5 are well-constructed, even when dealing with very low-dimensional latent variables.

\begin{figure}[t]
\centering
\includegraphics[width=0.9\columnwidth]{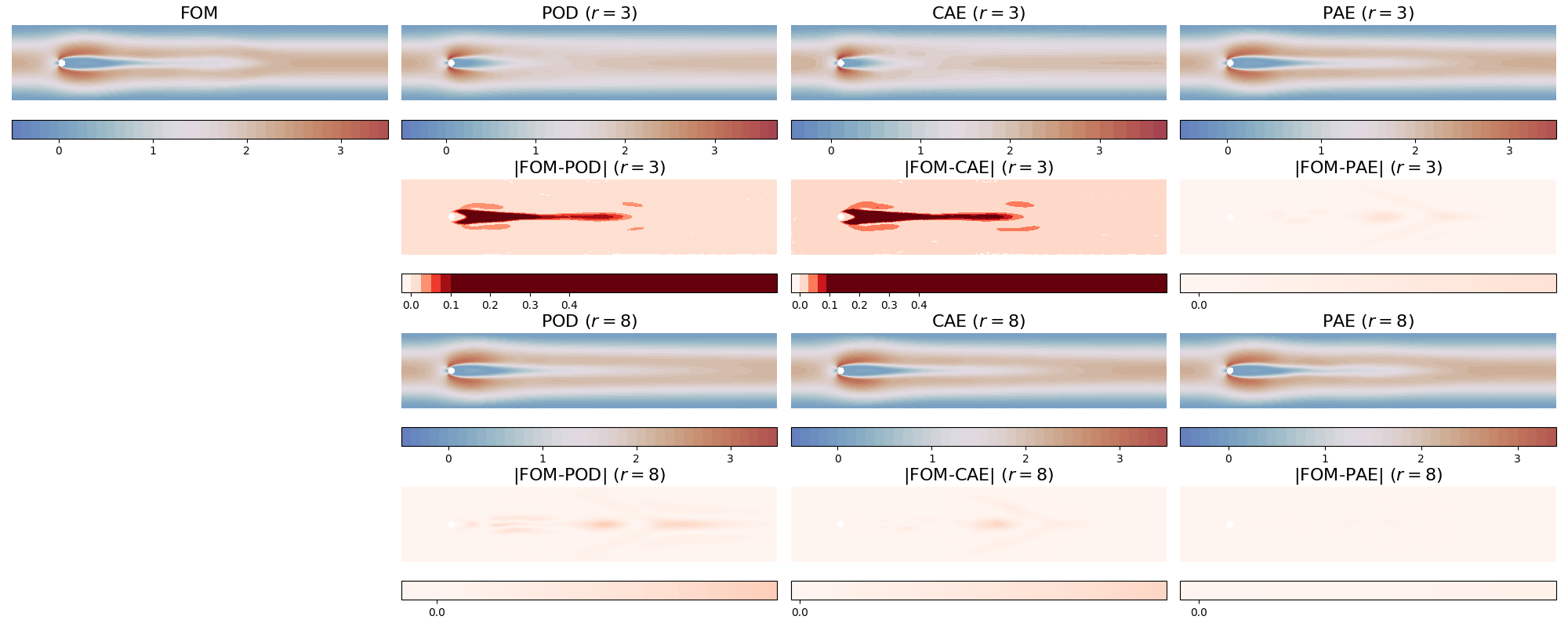}
\caption{Comparison of the reference generated by the full order model (FOM) and the developed snapshots of POD, CAE, and PAE at $t=2.0$: training session for the single cylinder case (\Cref{sec:single-pae}).}
\label{fig:single-snap100}
\end{figure}

\begin{figure}[t]
\centering
\includegraphics[width=0.9\columnwidth]{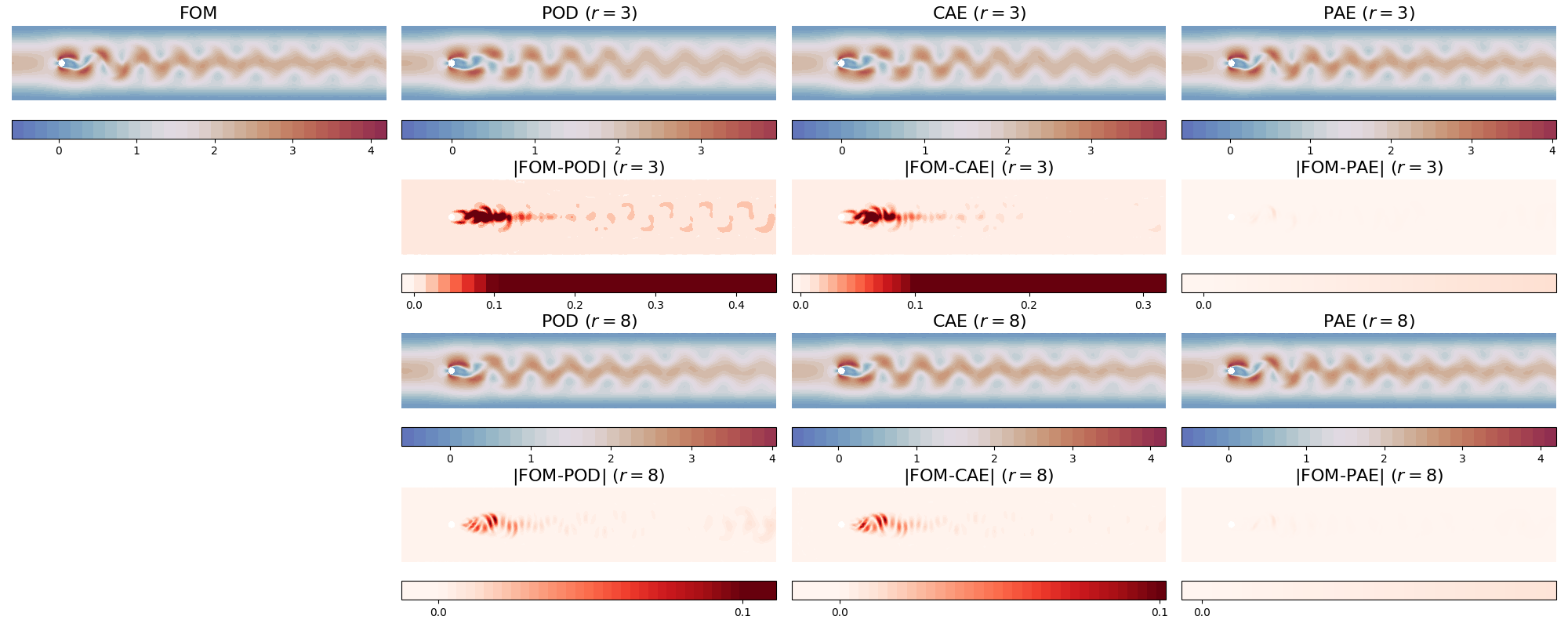}
\caption{Comparison of the reference generated by the full order model (FOM) and the developed snapshots of POD, CAE, and PAE at $t=14.0$: evaluation session for the single cylinder case (\Cref{sec:single-pae}).}%
\label{fig:single-snap700}
\end{figure}

\Cref{fig:single-snap100} and \Cref{fig:single-snap700} show a comparison of the developed snapshots from FOM, POD, CAE, and PAE3 at time $t=2.0, 14.0$ respectively when $r=3,8$.
The figures of their absolute errors show that PAE3 outperforms other models in terms of the state reconstruction with very low-dimensional parametrizations.

\begin{figure}[t]
\centering
\includegraphics[width=0.7\columnwidth]{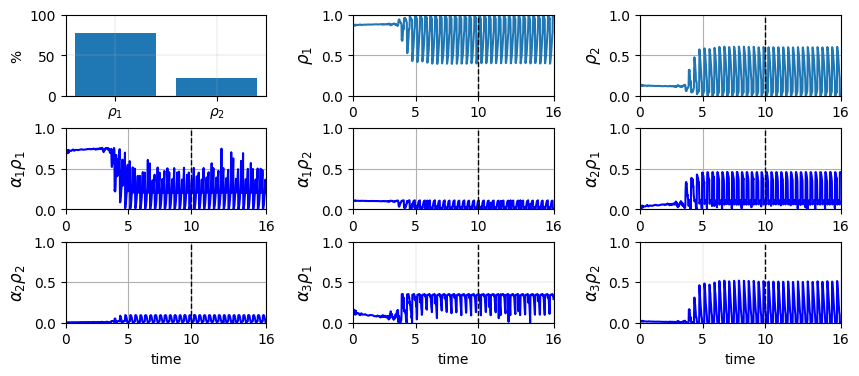}
\caption{Activation rates and trajectories of latent state variables when $r=2$: the dashed line separates the training and the extrapolation phases for the single cylinder case (\Cref{sec:single-pae}).}%
\label{fig:single-rho2}
\end{figure}

\begin{figure}[t]
\centering
\includegraphics[width=0.5\columnwidth]{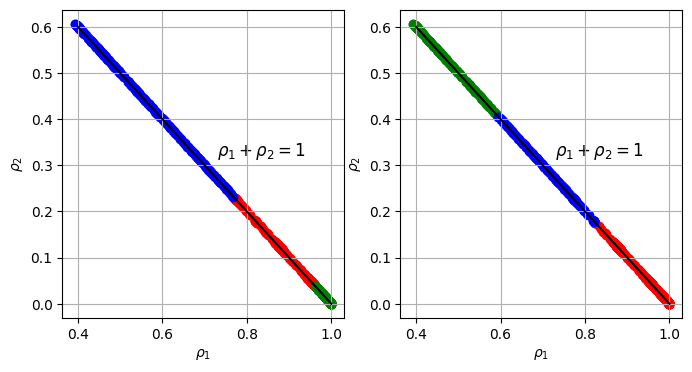}
\caption{Comparison of the results of (left) the smooth clustering $c(\brho)$ and (right)
$k$-means clustering with 3 clusters in the two-dimensional space for the single
cylinder case (\Cref{sec:single-pae}). The different
colors denote different clusters.}%
\label{fig:single-rho2d}
\end{figure}

\begin{figure}[t]
\centering
\includegraphics[width=0.6\columnwidth]{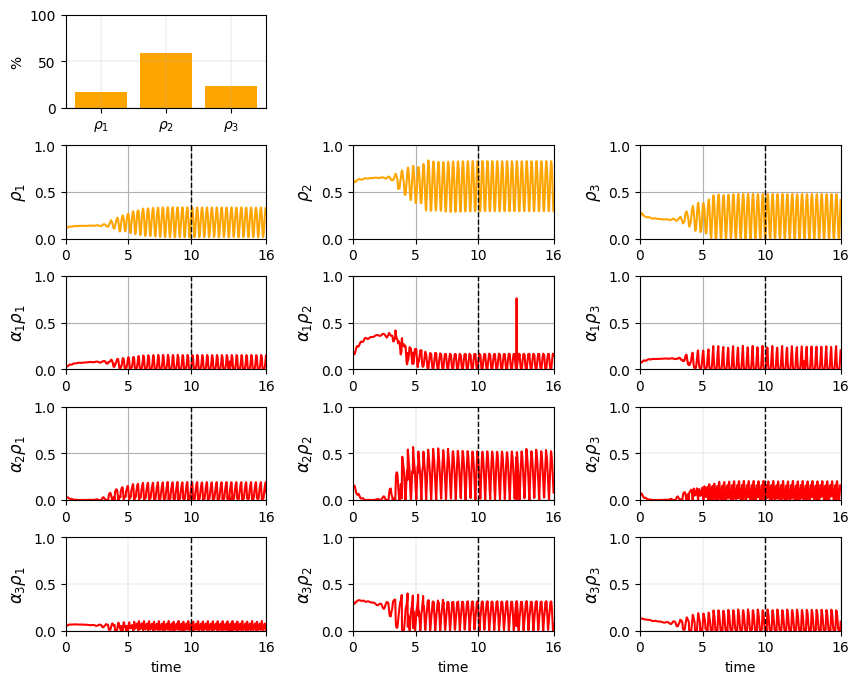}
\caption{Activation rates and trajectories of latent state variables when $r=3$: the dashed line separates the training and the extrapolation phases for the single cylinder case (\Cref{sec:single-pae}).}%
\label{fig:single-rho3}
\end{figure}

\begin{figure}[t]
\centering
\includegraphics[width=0.6\columnwidth]{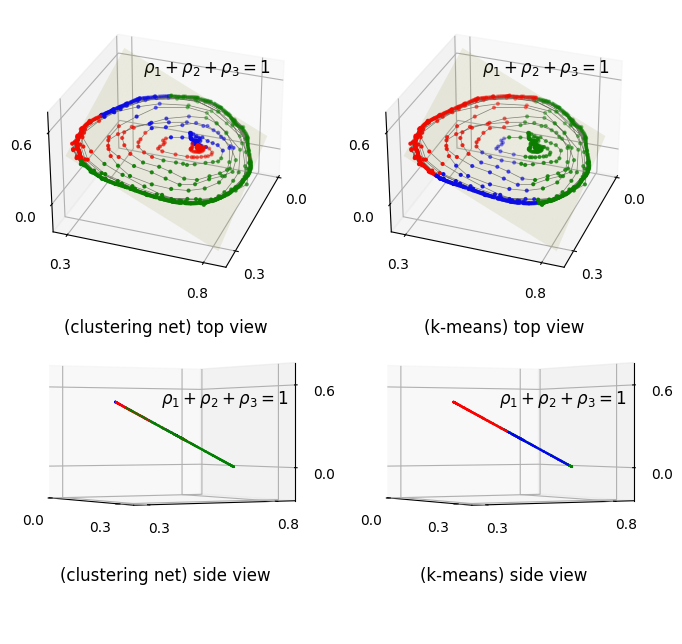}

\caption{Comparison of the results of (left) the smooth clustering $c(\brho)$ and (right)
$k$-means clustering with 3 clusters in the three-dimensional space for the single
cylinder case as described in \Cref{sec:single-pae}. The different
colors denote different clusters.}%
\label{fig:single-rho3d}
\end{figure}

\Cref{fig:single-rho2} displays the activation rates of the vertices for a polytope and the trajectories of latent state variables when $r=2$ with $k=3$ (i.e., PAE3).
It is shown that the latent variables for each PAE are within the range $[0,1]$ as the coefficients of a convex combination.

The activation rate is a metric indicating the relative extent to which each latent variable influences the state reconstruction. 
The activation rate of the $i$-th latent variable is defined as
\begin{equation*}
\frac{\sum_{j=1}^N \rho_{i,j}}{\sum_{l=1}^r \sum_{l=1}^N \rho_{l,j}}
\end{equation*}
where $N$ is the number of snapshots.
Since
\begin{subequations}
\begin{align*}
\frac{\sum_{j=1}^N (\alpha_1\rho_{i,j}+\cdots \alpha_k\rho_{i,j})}{\sum_{l=1}^r \sum_{l=1}^N (\alpha_1\rho_{l,j}+\cdots \alpha_k\rho_{l,j})}&=\frac{\sum_{j=1}^N (\alpha_1+\cdots \alpha_k)\rho_{i,j}}{\sum_{l=1}^r \sum_{l=1}^N (\alpha_1+\cdots \alpha_k)\rho_{l,j}}\\
&=\frac{\sum_{j=1}^N \rho_{i,j}}{\sum_{l=1}^r \sum_{l=1}^N \rho_{l,j}},
\end{align*}
\end{subequations}
the activation rate of the polytope coefficients related to the $i$-th latent variable is identical to the activation rate of the $i$-th latent variable.

As shown in the bar chart, when $r=2$, the state reconstruction overwhelmingly depends on the first latent variable accounting for $77.9\%$.
Consequently, the three vertices weighted by $\rho_1$ of the polytope significantly influence the state reconstruction with a rate of $77.9\%$.

\Cref{fig:single-rho2d} confirms that the latent variables satisfy the
convex combination constraints, $$\rho_1(t)+\rho_2(t)=1 \text{ and }
\rho_1(t),\rho_2(t)\geq 0$$
and the clustering net $c$ classifies latent variables similarly to $k$-means clustering.
In other words, any latent variables lie on the line $$\rho_1(t)+\rho_2(t)=1,\, \forall t>0.$$ when $r=2$.

\Cref{fig:single-rho3} shows the activation rates of the vertices for a polytope and the trajectories of latent state variables when $r=3$ with $k=3$ (i.e., PAE3).
For any $r$, the encoder ensures that all latent variables are nonnegative, and the summation of the elements for each $\brho(t)$ is 1. 
Thus, the trajectories of latent variables are within the expected range of $[0,1]$.
In the bar chart, the activation rates of each latent variable are $17.4\%$, $58.9\%$, and $23.7\%$.
Consequently, the three vertices weighted by $\rho_2(t)$ of the polytope have a large impact on the state reconstruction compared to the others.

\Cref{fig:single-rho3d} displays the distribution of latent variables in a three-dimensional space, comparing labels obtained by the clustering net with those from $k$-means clustering. 
As latent variables represent the coefficients of a polytope, they are nonnegative and lie on the plane $$\rho_1(t)+\rho_2(t)+\rho_3(t)=1,\, \forall t>0.$$

The clustering net tends to assign labels to latent variables in a manner similar to $k$-means clustering, as it utilizes pseudo-labels obtained from $k$-means clustering.
However, the clustering net is less constrained by distances from centroids due to its training with the joint loss function outlined in \nameref{sec:app-train}.

\subsection{Dataset: double cylinder}
The double cylinder setups features rich and chaotic dynamics even for
relatively low Reynolds numbers $\RE \in[40,100]$.
For our experiments we consider the flow at $\RE=60$ and generate training data
as laid out above.
Now, each snapshot vector $\bv(t)$ encompasses 46,014 states (i.e., $n=46014$), corresponding to nodes in the FEM mesh within the spatial domain $(-20,50)\times (-20,20)\subset\mathbb{R}^2$.

The states in the time domain $[0,240]$ 
As the initial value is rather unphysical and appeared to dominate the data in
an unfavourable fashion, we allowed the flow to evolve first and only considered
data points after the time $t=240$. 
Again, the dataset is partitioned into a training set, comprising 720 snapshots in the time interval $[240,480]$, and a test set, encompassing 432 snapshots within $[480, 760]$.

When employing convolutional encoders, the interpolation matrix $\bI_C\in\mathbb{R}^{46014\times 7505}$ transforms $\bv(t)\in\mathbb{R}^{42764}$ into $\bvcnn(t)\in\mathbb{R}^{2 \times 95 \times 79}$, representing the x-directional velocity and y-directional velocity at each mesh point on the rectangular grid of size  $95 \times 79$.

\subsection{PAEs: double cylinder}\label{sec:double-pae} 
\begin{table}[t]
\begin{adjustbox}{width=200pt,center} 
\begin{tabular}{c|c|c|c|c|c|c}
\textbf{Model} & $r$ & $L_e$ & $P_e$ & $L_d$ & $P_d$ & $R$ \\
\hline
POD & 2 & 1& 92,028  & 1& \textbf{92,028} & 4 \\ 
CAE & 2 & 27& \textbf{91,410} &1& \textbf{92,028}  & \textbf{2}  \\ 
PAE($k=3$) & 2 &27& \textbf{91,410} &3& 460,140 & 6\\
\hline
POD & 3 & 1& 138,042 &1& \textbf{138,042}  & 8 \\
CAE & 3 &27& \textbf{91,443} &1& \textbf{138,042}  & \textbf{3} \\
PAE($k=3$) & 3 &27& \textbf{91,443} &3& 690,266 & 9  \\
\hline
POD & 5 & 1& 230,070 &1& \textbf{230,070}  & 32 \\ 
CAE & 5 &27& \textbf{91,509} &1& \textbf{230,070}  & \textbf{5} \\ %
PAE($k=3$) & 5 &27& \textbf{91,509} &3& 690,210  & 15 \\
\hline
POD & 8 & 1& 368,112 &1& \textbf{368,112} & 256 \\ 
CAE & 8 &27& \textbf{91,608} &1& \textbf{368,112}  & \textbf{8} \\ %
PAE($k=3$) & 8 &27& \textbf{91,608} &3& 1,104,336 & 24 \\
\end{tabular}
\end{adjustbox}
\caption{Model information: the number of encoding ($L_e$) and decoding ($L_d$) layers, the number of encoding ($P_e$) and decoding ($P_d$) parameters, 
and the number
$R$ which counts the vertices of a bounding box in $\mathbb R^r$ (for POD) or of
the polytopes used for the reconstruction for CAE and PAE for the double cylinder case (\Cref{sec:double-pae});
see how to calculate $P_e$ and $P_d$ in \Cref{sec:measure}.}
\label{tab:double-params}
\end{table}

In this simulation, each PAE and CAE has a deep convolutional encoder consisting of 26 convolutional layers and a fully connected layer. 
The ELU activation function is applied to the convolutional layers and the modified softmax function (\ref{eq:msoftmax}) in the last layer.
To reduce the number of nodes in the last layer, the global average pooling is used before performing the fully connected computation.

In \Cref{tab:double-params} we tabulate the different model parameters.
Apparently, CAE and PAE maintain a relatively consistent number of encoding parameters regardless of reduced dimensions, in contrast with POD.
In practice, when $r=2,3,5,8$, the encoding size of CAE and PAE is only $99.3\%, 66.2\%, 39.8\%$, and $24.9\%$ of the encoding size of POD respectively.

According to their decoding sizes, for POD and CAE, their decoders are linear, resulting in sizes of $nr$, denoting the number of elements for a $n \times r$  matrix. 
the decoding size of PAE is $nrk+m$ where $m$ is the number of parameters in the clustering net.
In comparison with POD, CAE and PAE allow for the definition of a small number of $R$ vertices for polytopic LPV representations.

For training CAEs and PAEs, the Adam optimizer is used with a learning rate $\eta$ of $10^{-4}$, 
a batch size of 64, and a clustering loss weight $\lambda$ of $10^{-4}$; see \nameref{sec:app-train}.

\subsection{Results: double cylinder}
\begin{figure}[t]
\centering
\includegraphics[width=0.8\columnwidth]{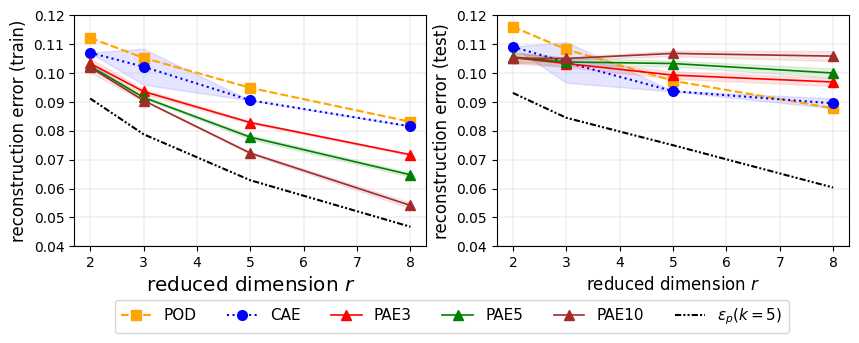}
\caption{Reconstruction error across the reduced dimension $r$ averaged for 5
runs for the double cylinder case (\Cref{sec:double-pae}). The shaded regions
mark the statistical uncertainty measured through several training runs and
appears to be insignificant and, thus, invisible in the plots for most methods.}
\label{fig:rec-double}
\end{figure}

\begin{figure}[t]
\centering
\includegraphics[width=0.9\columnwidth]{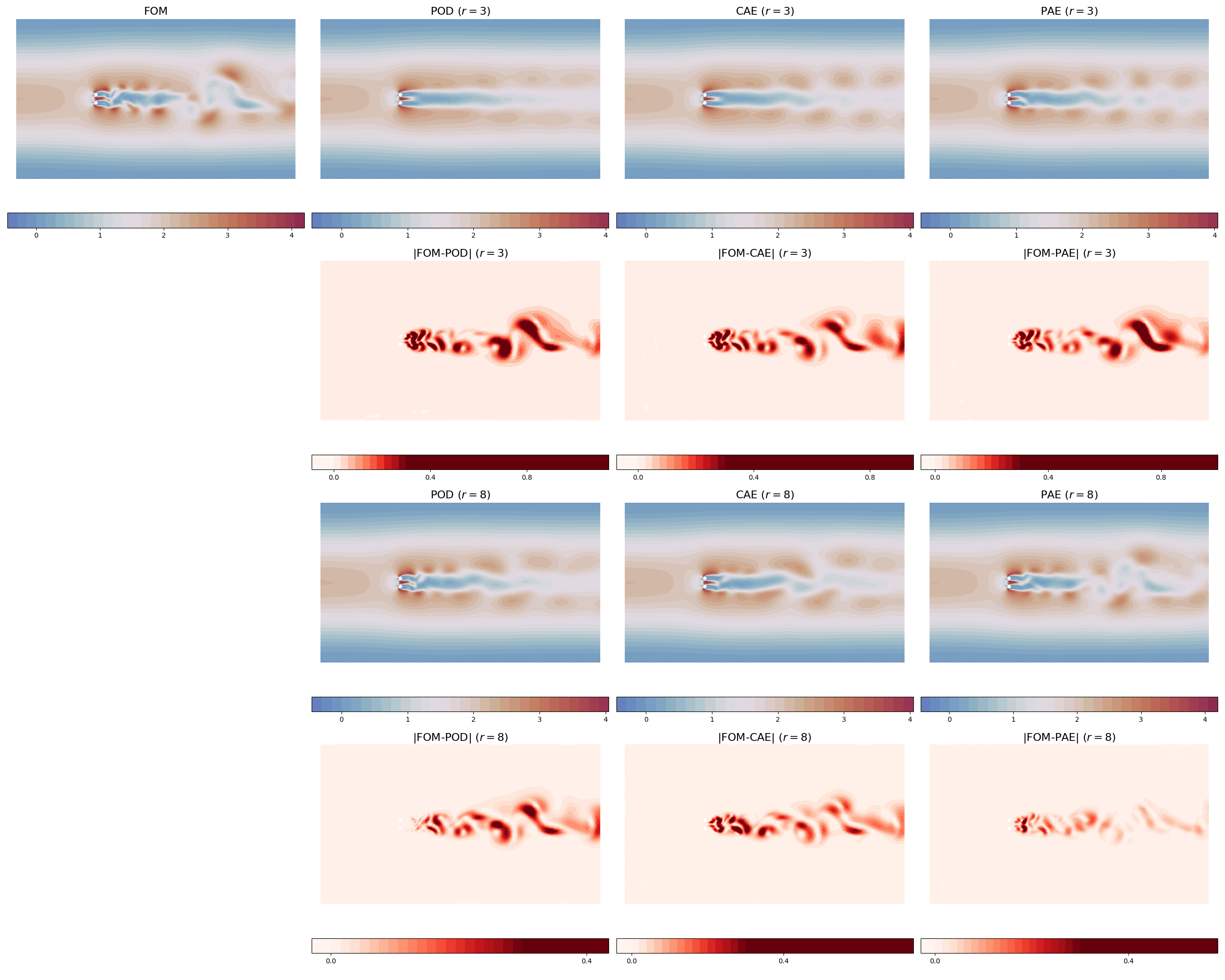}
\caption{Comparison of the reference generated by the full order model (FOM) and the developed snapshots of POD, CAE, and PAE at $t=556.0$: training session for the double cylinder case (\Cref{sec:double-pae}).}%
\label{fig:double-snap700}
\end{figure}

\begin{figure}[t]
\centering
\includegraphics[width=0.9\columnwidth]{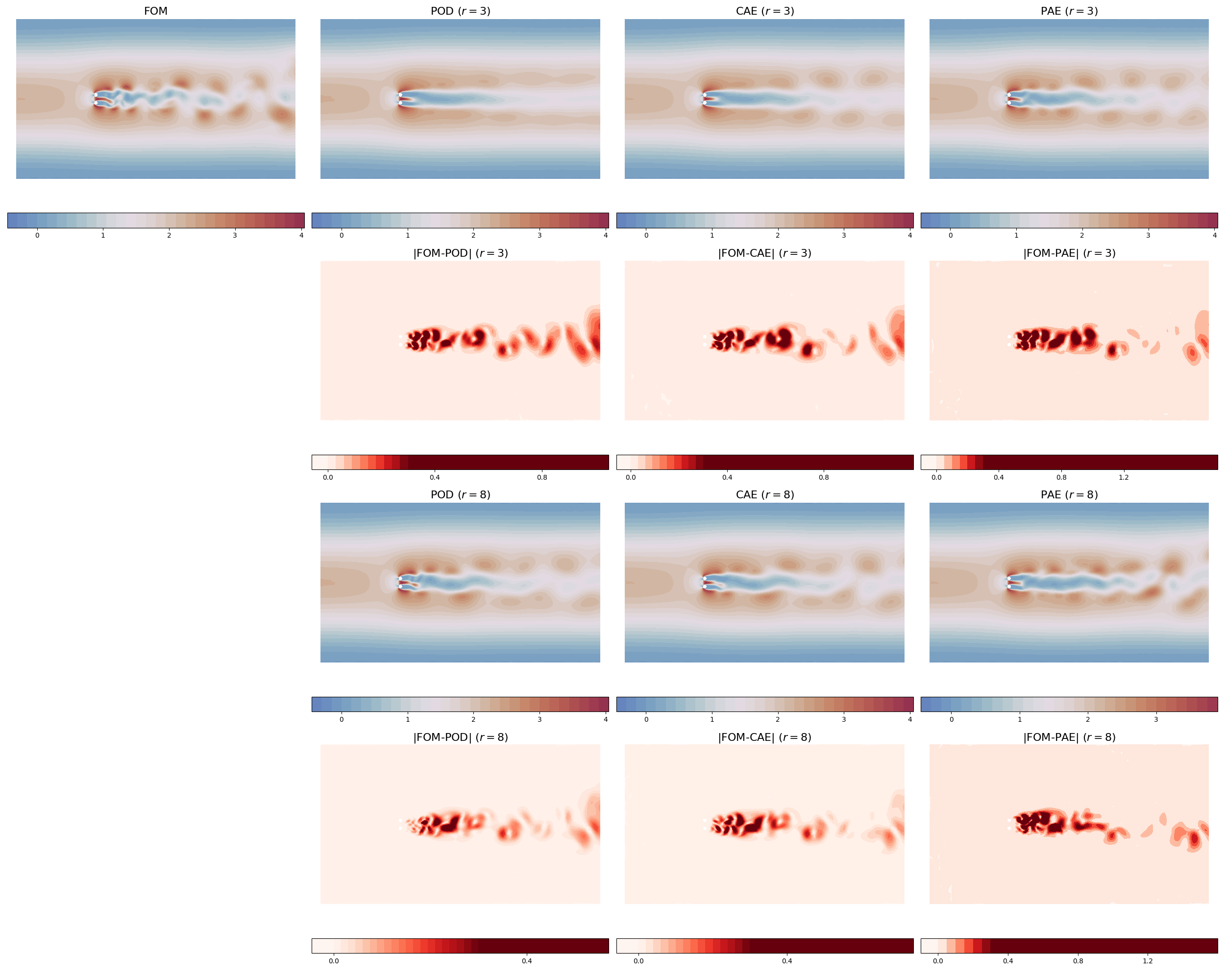}
\caption{Comparison of the reference generated by the full order model (FOM) and the developed snapshots of POD, CAE, and PAE at $t=736.5$: evaluation session for the double cylinder case (\Cref{sec:double-pae}).}%
\label{fig:double-snap1100}
\end{figure}

In \Cref{fig:rec-double}, PAE demonstrates superiority over POD and CAE across reduced dimensions in terms of the reconstruction performance for the training data. The reconstruction errors of PAE tend to significantly decrease as $k$ gets larger and $r$ increases.

However, PAEs fit the training data so well that the models lack generalization. 
Despite efforts to address overfitting through various regularization techniques including $L_1$ \emph{regularization}, $L_2$ \emph{regularization}, \emph{Dropout} \cite{NiGe14}, \emph{label smoothing} \cite{ChVi16}, and the application of fewer model parameters, the issue persists.

Nevertheless, the polytope errors $\varepsilon_p$ with $k=5$ reach a certain level in both the training and test phases.
It indicates that each polytope defined by PAEs is well-constructed.
Therefore, there is a potential to improve the model by tuning the encoding and the clustering parts involved in convex combination coefficients.

\Cref{fig:double-snap700} and \Cref{fig:double-snap1100} display a comparison of the developed snapshots from FOM, POD, CAE, and PAE3 at time $t=556.0, 736.5$ respectively when $r=3, 8$. 
At $t=556.0$, PAE3 reconstructs the state more clearly than the others.
In the reconstructed snapshot obtained from POD with $r=3$ at time $t=736.5$, the cylinder wake exhibits a notably smoother flow compared to CAE and PAE3. When $r=8$, there is no discernible difference among the applied methods.

\begin{figure}[t]
\centering
\includegraphics[width=0.6\columnwidth]{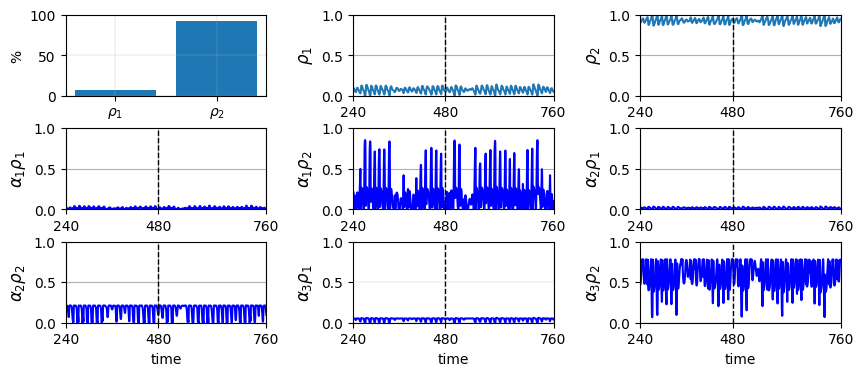}
\caption{Activation rates and trajectories of latent state variables when $r=2$: the dashed line separates the training and the extrapolation phases for the double cylinder case (\Cref{sec:double-pae}).}%
\label{fig:double-rho2}
\end{figure}

\begin{figure}[t]
\centering
\includegraphics[width=0.5\columnwidth]{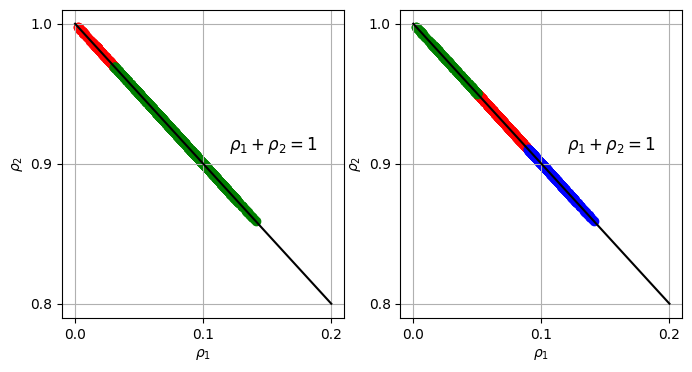}
\caption{Comparison of the results of (left) the smooth clustering $c(\brho)$ and (right)
$k$-means clustering with 3 clusters in the two-dimensional space for the double
cylinder case (\Cref{sec:double-pae}). The different
colors denote different clusters.}%
\label{fig:double-rho2d}
\end{figure}

In \Cref{fig:double-rho2}, it is evident that $\rho_2(t)$ exerts a predominant influence on state reconstruction, with an activation rate of $92.8\%$. 
\Cref{fig:double-rho2d} visually represents the latent variables satisfying convex combination constraints.
Moreover, the clustering net $c$ classifies latent variables in a manner reminiscent of $k$-means clustering. 
However, the clustering net assigns only two labels to the latent variables, in contrast to $k$-means clustering.

\begin{figure}[t]
\centering
\includegraphics[width=0.6\columnwidth]{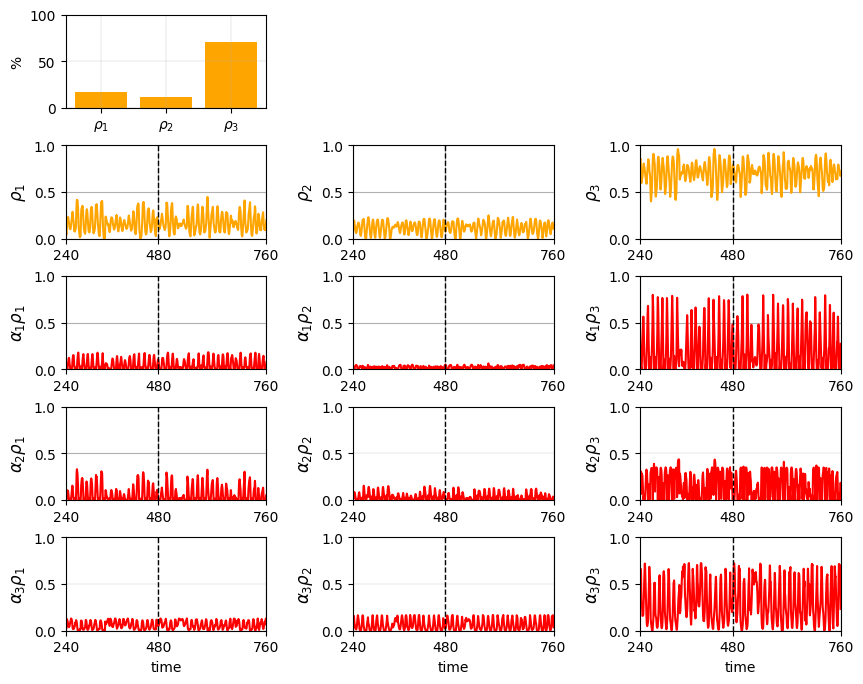}
\caption{Activation rates and trajectories of latent state variables when $r=3$: the dashed line separates the training and the extrapolation phases for the double cylinder case (\Cref{sec:double-pae}).}%
\label{fig:double-rho3}
\end{figure}

\begin{figure}[t]
\centering
\includegraphics[width=0.6\columnwidth]{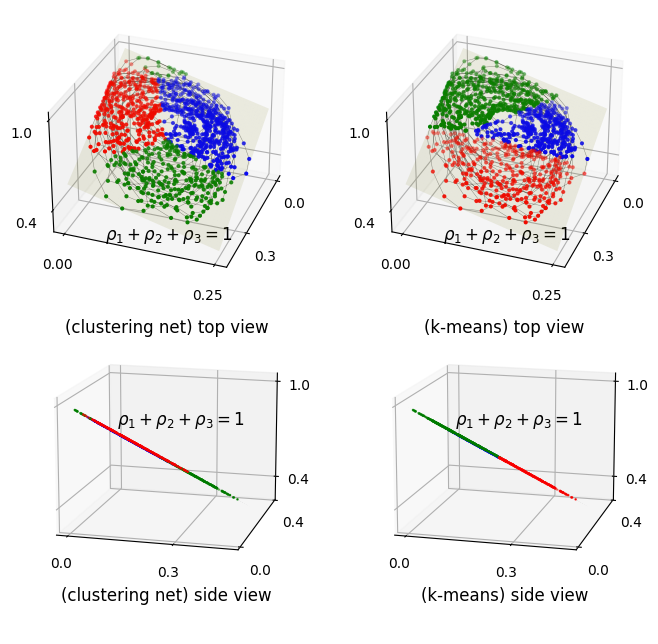}
\caption{Comparison of the results of (left) the smooth clustering $c(\brho)$ and (right)
$k$-means clustering with 3 clusters in the three-dimensional space for the double
cylinder case (\Cref{sec:double-pae}). The different
colors denote different clusters.}%
\label{fig:double-rho3d}
\end{figure}

In \Cref{fig:double-rho3}, the activation rates are $17.6\%$, $11.8\%$, and $70.6\%$ respectively.
As a result, the substantial influence on state reconstruction comes from the three vertices that are weighted by $\rho_3(t)$ in the polytope surpassing the impact of others.

\Cref{fig:double-rho3d} shows the trajectory of the latent variables on the plane $$\rho_1(t)+\rho_2(t)+\rho_3(t)=1,\, \forall t>0.$$
Overall, The clustering net tends to classify latent variables in a manner similar to $k$-means clustering.
Nevertheless, in certain instances, the clustering net assigns a label to them in defiance of distance-based methods.

\clearpage
\section{Conclusion}\label{sec:concl}
In this article, we proposed a polytopic autoencoder architecture consisting of a lightweight nonlinear encoder, a convex combination decoder, and a differentiable clustering network.
We also showed how a differentiable clustering network embedded in the decoder improves state reconstruction errors. 
Notably, the model ensures that all reconstructed states reside within a polytope, 
with their latent variables serving directly as the convex combination coefficients.

To estimate the optimal performance of polytopes obtained by the proposed model, we measured polytope errors.
From another perspective, we investigated the dominance of vertices in state reconstruction within a polytope by utilizing the activation rates of latent variables.

The results of the single cylinder case indicated that our model well
outperforms POD in terms of reconstruction.

For the more challenging dynamics of the double cylinder, despite a rather low
error in the training regime and a rather low potential model mismatch as
estimated by the approximation error in the polytope, the reconstruction
deteriorated in the extrapolation regime.
Nevertheless, we confirmed that the polytope is well-constructed in all the regimes.
Thus, a potential avenue for further research is to enable this potential by 
 revising the architecture or training strategy e.g. by including residuals in
 the loss functions or by extending the networks to better handle unseen
 clusters.
 
\section*{Acknowledgments}%
\addcontentsline{toc}{section}{Acknowledgments}
We acknowledge funding by the German Research Foundation (DFG) through the research training group 2297 ``MathCoRe'', Magdeburg.

\bibliographystyle{abbrvurl}
\bibliography{pae}

\section*{Appendix A}\label{sec:app-cv}
The well-posedness of the polytope error (Definition \ref{def:pe}) can be derived
from the following three lemmas:
\begin{lemma}\label{lem:ineqm}
Let $\bA$ be a positive definite matrix. Then 
\begin{equation*}
\bx^\top\bA\by\leq\sqrt{\bx^\top\bA\bx}\sqrt{\by^\top\bA\by}
\end{equation*}
for any nonzero vectors $\bx$ and $\by$.
\end{lemma}
\begin{proof}
Since $\bA$ is positive definite, we obtain the inequality
\begin{subequations}
\begin{align*}
(\bx-a\by)^\top\bA(\bx-a\by)&\geq0\\
\rightarrow \bx^\top\bA\bx-2a\bx^\top\bA\by+a^2\by^\top\bA\by&\geq0\\
\end{align*}
\end{subequations}
where $a$ is a scalar. Let
\begin{equation*}
p(a)= (\by^\top\bA\by)a^2-2(\bx^\top\bA\by)a+(\bx^\top\bA\bx).
\end{equation*}
Then we get $p(a)\geq0$. Hence, the discriminant of $p(a)$ is less than 0 as follows:
\begin{subequations}
\begin{align*}
(\bx^\top\bA\by)^2-(\bx^\top\bA\bx)(\by^\top\bA\by)\leq0\\
\rightarrow(\bx^\top\bA\by)^2\leq(\bx^\top\bA\bx)(\by^\top\bA\by)
\end{align*}
\end{subequations}
\begin{equation*}
\rightarrow -\sqrt{\bx^\top\bA\bx}\sqrt{\by^\top\bA\by}\leq\bx^\top\bA\by\leq\sqrt{\bx^\top\bA\bx}\sqrt{\by^\top\bA\by}
\end{equation*}
Thus,
\begin{equation*}
\bx^\top\bA\by\leq\sqrt{\bx^\top\bA\bx}\sqrt{\by^\top\bA\by}.
\end{equation*}
When $\by\neq k\bx$ with a scalar $k$,
\begin{equation*}
\bx^\top\bA\by< \sqrt{\bx^\top\bA\bx}\sqrt{\by^\top\bA\by}.
\end{equation*}
\end{proof}

\begin{lemma}\label{lem:depinpoly}
Let $\bx$ and $\by$ be vectors in a convex polytope $\tilde{\mathcal{V}}$. If $\bx\neq\by$ then $\by\neq k\bx$ with a scalar $k$.
\end{lemma}
\begin{proof}
We deal with the contrapositive of the above statement.
Suppose that $\by=k\bx$.
Then $\exists$ $c_i\geq 0$, $\forall i$, and $\sum_i c_i=1$ such that
\begin{equation*}
\bx=\sum_i c_i\bu_i\in\tilde{\mathcal{V}},\, \by=k\sum_i c_i\bu_i
\end{equation*}
where $\bu_i$ is the $i$-th vertex of $\tilde{\mathcal{V}}$.
Since $\by\in\tilde{\mathcal{V}}$, 
\begin{equation*}
k\sum_i c_i=1\, (\text{i.e. } k=1).
\end{equation*}
Thus, $\bx=\by$.
\end{proof}

\begin{lemma}\label{lem:dist}
Let $\tilde{\mathcal{V}}\subset \mathbb{R}^n$ be a convex polytope. Then for any $\bv\in\mathbb R^{n}\setminus  \tilde{\mathcal{V}}$ there exists a unique $\bv^\ast \in \tilde{\mathcal{V}}$ such that
  \begin{equation*}
    \|\bv-\bv^\ast \|_\bM = \min\{\|\bv-\tilde{\bv}\|_\bM\colon \tilde{\bv} \in \tilde{\mathcal{V}}\}.
  \end{equation*}
\end{lemma}
\begin{proof}
Define
\begin{equation*}
f(\bx)=\|\bx\|_\bM, \bx\in\tilde{\mathcal{V}}.
\end{equation*}
We first consider the existence of $\bv^\ast$. Since $\tilde{\mathcal{V}}$ are compact and $f$ is continuous, $\exists \bv^\ast\in\tilde{\mathcal{V}}$ such that 
\begin{equation*}
f( \bv^\ast)\leq f(\bx), \forall \bx\in\tilde{\mathcal{V}}.
\end{equation*}
To show the uniqueness of $\bv^\ast$, we first prove that $f$ is strictly convex.\\
\begin{subequations}
\begin{align*}
f(\lambda\bx+(1-\lambda)\by) =& \sqrt{(\lambda\bx+(1-\lambda)\by)^\top\bM(\lambda\bx+(1-\lambda)\by)}\\
=& \sqrt{(\lambda^2\{f(\bx)\}^2+2\lambda(1-\lambda)\bx^\top\bM\by+(1-\lambda)^2f(\by)\}^2}\\
<& \sqrt{(\lambda^2\{f(\bx)\}^2+2\lambda(1-\lambda)f(\bx)f(\by)+(1-\lambda)^2\{f(\by)\}^2}\, (\because \text{Lemmas \ref{lem:ineqm}, \ref{lem:depinpoly}})\\
=& \sqrt{(\lambda f(\bx)+(1-\lambda)f(\by)})^2\\
=& \lambda f(\bx)+(1-\lambda)f(\by)
\end{align*}
\end{subequations}
where $0<\lambda<1$, $\bx\neq\by$, and $\forall \bx,\by\in\tilde{\mathcal{V}}$.\\
Hence, we obtain the inequality
\begin{equation*}
f(\lambda\bx+(1-\lambda)\by) < \lambda f(\bx)+(1-\lambda)f(\by)
\end{equation*}
Therefore, $f$ is strictly convex.\\
To show the uniqueness of $\bv^\ast$ by contradiction,
suppose $\exists \bw\in\tilde{\mathcal{V}}, \bv^\ast\neq\bw$ such that
\begin{equation*}
f(\bv^\ast)=f(\bw)
\end{equation*}
Since $f$ is strictly convex and $\bv^\ast\neq\bw$,
\begin{equation*}
f(\lambda\bv^\ast+(1-\lambda)\bw) < \lambda f(\bv^\ast)+(1-\lambda)f(\bw), 0<\lambda<1.
\end{equation*}
Since $f( \bv^\ast)= f(\bw)$,
\begin{subequations}
\begin{align*}
f(\lambda\bv^\ast+(1-\lambda)\bw) <& \lambda f(\bv^\ast)+(1-\lambda)f(\bw)\\
=&  \lambda f(\bv^\ast)+(1-\lambda)f(\bv^\ast)\\
=& f(\bv^\ast)
\end{align*}
\end{subequations}
\begin{equation*}
\therefore f(\bv^\ast)> f(\lambda\bv^\ast+(1-\lambda)\bw).
\end{equation*}
Then it contradicts $f(\bv^\ast)$ is a minimum of $f$.
Thus, $\bv^\ast=\bw$. (i.e. $\bv^\ast$ is unique.)\\
Since $\|\bx\|_\bM$ has a unique minimum, a parallel translation of the norm, $\|\bv-\bx \|_\bM$, also has a unique minimum.
\end{proof}

\section*{Appendix B}\label{sec:app-dsc}
A depthwise convolution involves applying a $K \times K$ convolution operation to each input channel independently, resulting in a total number of parameters equal to $K \times K \times 1 \times C_I$. The convolution captures spatial information within each input channel.
Pointwise convolution is a convolution with a kernel size of $1\times 1\times C_O \times C_I$. It combines the output from the depthwise convolution to create the final feature map.
Thus, the total number of parameters is  $K \times K \times C_I + C_O \times C_I$, resulting in a parameter reduction rate denoted as 
\begin{equation*}
\frac{K \times K \times C_I + C_O \times C_I}{K\times K\times C_O \times C_I}=\frac{1}{C_O}+\frac{1}{K^2}.
\end{equation*}
For example, when $K=3$ and $C_O=8$, the convolution parameters are only $23\%$ of those in the $3\times 3\times 8 \times C_I$ standard convolution.
Finally, an encoder $\mu$ contains deep inverted residual blocks:
\begin{subequations}\label{eq:en-deeplayer}
\begin{align*}
&h_1=a(\textsc{CV}_{3\times 3}(\bvcnn))\\
&h_l=\textsc{RB}(h_{l-1}), l=2,\cdots , L-1\\
&h'_{L-1}=\textsc{GAP}(h_{L-1})\\
&h_L = \sigma(\textsc{LN}(h'_{L-1}))
\end{align*}
\end{subequations}
where $a(\cdot)$, $CV_{3\times 3}$, $\textsc{RB}$, $\textsc{GAP}$, and \textsc{LN} are a nonlinear activation function, 
a $3\times 3$ standard convolution, a inverted residual block, a global average pooling, and a linear layer respectively. 
$\textsc{RB}$ generates feature maps that match the size of the inputs without including any bias terms to minimize the number of model parameters, 
except for downsampling in specific layers. 
When downsampling occurs, $\textsc{RB}$ omits computing a residual connection, reducing the size of feature maps by adjusting the stride from 1 to 2. 
Therefore, the final $C_{final}\times H_{final}\times W_{final}$ feature map $h_{L-1}$ is obtained 
where $(H_{final}\times W_{final})$ is  a smaller size, but a channel dimension $C_{final}$ is larger than the input size $C \times H\times  W$.

\section*{Appendix C}\label{sec:app-train}

\begin{algorithm}[t]
\caption{Three-step training strategy for PAEs}\label{alg:pae}
\begin{algorithmic}
\State \textbf{Step 0. Prepare the Model} 
\State Set hyperparameters like $r$, $k$, and the
number of epochs $N_1$, $N_2$, $N_3$
\State Initialize model parameters $\theta$ for the corresponding models
(e.g., $\theta_\mu$, $\theta_\varphi$, $\theta_c$). Note that, e.g., $\brho =
\mu(\theta_\mu; \bv)$ denotes the evaluation with the current values of the
parameters
\State Generate CNN inputs $\bvcnn=\bI_C\bv$
\State \textbf{Step 1. Train a CAE with encoder $\mu$ and decoder $\bar\varphi$} 
\For{$e=1,\cdots,N_1$}
\For{batch iterations: $i \in B_1$}
\State Compute $\brho = \mu(\theta_{\mu};\bvcnn)$
\State Compute $\tilde \bv =\bar\varphi(\theta_{\varphi};\brho)$
\State Compute $\Lrec(\theta_{\mu},\theta_{\bar\varphi};\bvcnn) = \frac{1}{|B_1|}\sum_{i\in B_1}\parallel \tilde \bv(\theta_{\mu},\theta_{\bar\varphi};{\bvcnn}^{(i)}) - \bv^{(i)}\parallel_\bM$
\State Update model parameters $\theta_\mu$ and $\theta_{\bar\varphi}$
\EndFor
\EndFor
\State \textbf{Step 2. Train individual decoders, $\varphi_{1}, \cdots, \varphi_{k}$} 
\State Freeze the weights $\theta_\mu$ of $\mu$
\For{$l=1,\cdots,k$}
\State Select small batch data based on partial $k$-means clustering and save their labels $\bl$
\For{$e=1,\cdots,N_2$}
\For{batch iterations: $i\in B_2$}
\State Compute $\brho = \mu(\bvcnn)$
\State Compute $\tilde \bv =\varphi^{(l)}(\theta_{\varphi^{(l)}};\brho)$
\State Compute $\Lrec(\theta_{\varphi^{(l)}};\brho) = \frac{1}{|B_2|}\sum_{i\in
B_2}\| \tilde \bv(\theta_{\varphi^{(l)}};{\brho}^{(i)}) - \bv^{(i)}\|_\bM$ 
\State Update model parameters $\theta_{\varphi_l}$
\EndFor
\EndFor
\EndFor
\State \textbf{Step 3. Train a PAE containing $\mu$, $\varphi$, and $c$} 
\State Define a matrix $\bU=[\theta_{\varphi_{1}}, \cdots, \theta_{\varphi_{k}}]$ ($\theta_{\varphi}= \{ \theta_{\varphi_{1}}, \cdots, \theta_{\varphi_{k}} \}$)
\For{$e=1,\cdots,N_3$}
\For{batch iterations: $i\in B_3$}
\State Compute $\brho = \mu(\theta_{\mu};\bvcnn)$
\State Compute $\tilde \bv =\varphi(\theta_{\varphi};\brho)=\bU(c(\theta_c;\brho) \otimes \brho)$
\State Compute $\Lrec(\theta_{\mu},\theta_{\varphi},\theta_c;\brho) = \frac{1}{|B_3|}\sum_{i\in B_3}\parallel \tilde \bv(\theta_{\mu},\theta_{\varphi},\theta_c;{\brho}^{(i)}) - \bv^{(i)}\parallel_\bM$
\State Compute $\Lclt(\theta_{\mu},\theta_c;\brho) =-\frac{1}{|P|}\sum_{j\in P\subset B_3} \bl^{(j)}\cdot \log(c(\theta_{\mu},\theta_c;{\brho}^{(j)}))$
\State Compute
$\mathcal{L}=\Lrec(\theta_{\mu},\theta_{\varphi},\theta_c;\brho)+10^{-4}\Lclt(\theta_{\mu},\theta_c;\brho)$
\State Update model parameters $\theta_{\mu}$, $\theta_{\varphi}$, and $\theta_c$
\EndFor
\EndFor
\end{algorithmic}
\end{algorithm}

We train PAEs through the following three steps (Algorithm \ref{alg:pae}):
\begin{itemize}
  \item \textbf{Step 1}: Initialization and identification of the latent space
  \item \textbf{Step 2}: Construction of polytopes for state reconstruction in each cluster
  \item \textbf{Step 3}: Construction of a polytope for state reconstruction
\end{itemize}
See \Cref{sec:tr} for explanations and definitions of the involved
quantities.

\end{document}